\documentclass[preprint,12pt]{article}
\usepackage[top=1.1in, bottom=1.2in, left=0.8in, right=0.8in]{geometry}
\usepackage{amsmath,amsfonts,amssymb,amsthm,mathtools}
\usepackage{enumitem}
\usepackage{mathrsfs}
\usepackage{graphicx}
\usepackage{colortbl,dcolumn}
\usepackage{psfrag}
\usepackage{booktabs}
\usepackage{tabularx}
\usepackage{array}
\newcolumntype{Y}{>{\raggedright\arraybackslash}X}
\usepackage{float}
\usepackage{cite}
\usepackage{url}
\usepackage{xcolor}
\usepackage{hyperref}
\allowdisplaybreaks[4]
\numberwithin{equation}{section}

\hypersetup{
  colorlinks=false,
  linkbordercolor={0.75 0.15 0.15},
  citebordercolor={0.10 0.60 0.10},
  urlbordercolor={0.10 0.45 0.75},
  pdfborder={0 0 0.7}
}

\newcommand{\R}{\mathbb{R}}

\newcommand{\E}{\mathbb{E}}

\newcommand{\op}{\mathrm{op}}
\newcommand{\dd}{\,\mathrm{d}}

\newcommand{\piz}{\pi_z}
\newcommand{\pieta}{\pi_{\eta}}
\newcommand{\Qeta}{Q_\eta}
\newcommand{\Qsg}{Q_\eta^{\mathrm{sg}}}

\newcommand{\Lgen}{\mathcal L_z}

\newtheorem{theorem}{Theorem}[section]
\newtheorem{definition}[theorem]{Definition}
\newtheorem{lemma}[theorem]{Lemma}
\newtheorem{proposition}[theorem]{Proposition}
\newtheorem{corollary}[theorem]{Corollary}
\newtheorem{remark}[theorem]{Remark}
\newtheorem{example}[theorem]{Example}
\newtheorem{assumption}[theorem]{Assumption}

\begin{document}
\title{Deterministic Envelopes for Tamed SGLD:\\
Decoupling Stochastic Gradient Noise and Localized Taming}

\author{
Yiwei Zhou\footnotemark[1], Ziheng Chen\footnotemark[2]
}

\maketitle

\footnotetext{\footnotemark[1] School of Mathematics and Statistics, Yunnan University, Kunming, Yunnan, 650500, China. Email: yiwei.zhou@utexas.edu}
\footnotetext{\footnotemark[2] School of Mathematics and Statistics, Yunnan University, Kunming, Yunnan, 650500, China. Email: 12024113103@stu.ynu.edu.cn}

\begin{abstract}
      {\rm\small
Stochastic gradient Langevin algorithms often use tamed denominators to stabilize superlinear drifts. This paper shows that when the denominator depends on the current stochastic gradient, the transformed update can have a biased conditional mean even if the original stochastic gradient is unbiased. This creates a stationary mean-shift channel that is absent for deterministic denominators.
We propose a structure-preserving framework for designing tamed denominators. The construction keeps the denominator deterministic given the current state, and uses localized deterministic envelopes to avoid unnecessary taming in typical regions. These kernels retain the stabilizing effect of taming while avoiding the bias introduced by a gradient-dependent denominator. Our theory bounds the stationary bias through Euler, envelope, and stochastic-gradient residuals. The analysis also shows why purely local taming rules can lose control in the far tail and motivates a hybrid construction with additional tail protection. Experiments confirm the stationary distortions of random denominators, the bias reduction of deterministic-envelope designs, and the stabilizing effect of the hybrid construction.}\\[0.5em]

      \textbf{AMS subject classification: } {\rm\small Primary 65C30; Secondary 60H10, 60J22, 60J60, 62F15.}\\[0.3em]

      \textbf{Key Words: } {\rm\small Tamed stochastic gradient Langevin dynamics; deterministic envelopes; localized taming; stationary bias; invariant measure.}
\end{abstract}


\section{Introduction}
In many tamed SGLD schemes, the denominator is computed from the current stochastic gradient.  This is usually presented as a way to shrink large steps.  The key point of this paper is that the shrinkage is not neutral.  Even if \(g(x,U)\) is an unbiased estimator of the drift \(b(x)\), the ratio
\[
\frac{g(x,U)}{1+\eta^\alpha A(x,U)}
\]
can have a biased conditional mean.  Thus the denominator can change the one-step mean drift before any long-time sampling error is considered.

This effect is the starting point of the paper.  We study it for
fixed-stepsize stochastic gradient Langevin dynamics.  Given data
\[
 z=(z_1,\ldots,z_N), \qquad F_z(w)=\frac1N\sum_{i=1}^N f(w,z_i),
\]
the reference overdamped Langevin diffusion is
\begin{equation}\label{eq:langevin}
 \dd X_t=-\nabla F_z(X_t)\dd t+\sqrt{2\beta^{-1}}\dd B_t,
\end{equation}
with Gibbs invariant law
\begin{equation}\label{eq:gibbs}
 \piz(\dd w)=\Lambda_z^{-1}e^{-\beta F_z(w)}\dd w.
\end{equation}
For polynomial-growth objectives, fixed-stepsize Euler simulation can be
unstable even when the continuous-time drift is dissipative.  For example,
\[
 F(x)=\frac14x^4-\frac12x^2,
 \qquad F'(x)=x^3-x,
\]
has far-field Euler increments that can overshoot.  This classical instability
explains why taming is useful.  It does not determine what information the
denominator should use.

The first design rule is to keep the denominator fixed conditional on the
current state.  For denominators of the form
\[
  1+\eta^\alpha A_\eta(x),
\]
where \(A_\eta(x)\) is deterministic given \(x\), an unbiased stochastic-gradient
oracle satisfies
\[
  \E\left[
  \frac{g(x,U)}{1+\eta^\alpha A_\eta(x)}\mid x
  \right]
  =\frac{\nabla F_z(x)}{1+\eta^\alpha A_\eta(x)}.
\]
The denominator still modifies the drift, but it does not create an additional
mean-shift channel through the current oracle draw.  This is the deterministic
envelope principle.

Deterministic taming is not the whole design problem.  A global deterministic
denominator can be stable but too conservative.  It may damp the drift even in
regions where the Euler step is already well behaved.  We therefore separate
the dependence question from the activation question.  The denominator should
be deterministic conditional on \(x\), but its correction should turn on only
where stabilization is needed.  This leads to the localized template
\begin{equation}\label{eq:intro-hybrid-envelope}
  A_\eta^{c_s,c_h}(x)
  =c_s\big(\bar A(x)-R\big)_+^\theta
   +c_h\big(\bar A(x)-S_\eta\big)_+,
  \qquad c_s,c_h\in\{0,1\},\quad 0<\theta\le1,
\end{equation}
where \(\bar A\) is a deterministic growth indicator.  The choices
\((c_s,c_h)=(1,0),(0,1),(1,1)\) give local soft, local hard, and hybrid
envelopes.  These are not different oracle principles.  They are different
activation patterns for the same deterministic-envelope idea.  A global
envelope, by contrast, uses the growth indicator everywhere, for example
\(A(x)=c\bar A(x)\) with fixed \(c>0\).

The results follow the same logic.
\begin{itemize}
\item We quantify the mean shift caused by gradient-dependent denominators.  In
one dimension it is driven by oracle variance.  In higher dimensions it can
also rotate the averaged drift through the oracle covariance.

\item We formulate deterministic envelopes for fixed-step Langevin kernels.
Global, local soft, local hard, and hybrid rules are treated as activation
patterns of one denominator-design framework.

\item We prove Lyapunov stability, invariant-measure existence, and
small-stepsize stationary consistency.  For Poisson observables, the
stationary bias is controlled by one-step residuals.

\item We test the design principle on radial examples and empirical-risk
examples.  The experiments isolate the mean-shift channel of random
denominators, the typical-region damping of global deterministic taming, and
the tail protection supplied by hybrid envelopes.
\end{itemize}

\paragraph{Relation to existing tamed schemes.}
Taming and truncation are standard tools for SDEs with non-globally Lipschitz
coefficients \cite{HutzenthalerJentzenKloeden2011,Hutzenthaler2012,Sabanis2013,Sabanis2016,Mao2015}
and for Langevin algorithms with superlinear potentials
\cite{DurmusMoulines2017,DalalyanKaragulyan2019,BrosseDurmusMoulinesSabanis2019}.
TULA-style methods provide deterministic drift stabilization.  TUSLA-type
analyses use deterministic or state-dependent taming factors in many cases and
provide finite-time guarantees for stabilized stochastic gradient Langevin
dynamics, including sampling, optimization, and nonconvex-learning bounds
\cite{LovasLytrasRasonyiSabanis2023,LimNeufeldSabanisZhang2024}.

The present paper studies a different issue.  It first identifies the oracle
mean shift induced by gradient-dependent denominators.  It then treats
deterministic denominators as fixed-step design objects and decomposes the
stationary bias into Euler, envelope, and stochastic-gradient residuals.
Gradient-dependent denominators are not only a toy baseline.  They appear, for
example, in normalized stochastic-gradient updates of the form
\[
    \frac{\nabla f(\xi_n,w_n)}{1+\alpha_n\|\nabla f(\xi_n,w_n)\|},
\]
where the normalization uses the current stochastic-gradient realization
\cite{EisenmannStillfjord2022}. This conditional-mean effect may be less visible in optimization, where the goal is often stable descent rather than preservation of a target law.  In sampling, it is more delicate: changing the conditional mean drift changes the Markov
kernel and can lead to stationary bias. 

Table~\ref{tab:comparison} summarizes the
comparison.

\begin{table}[H]
\begin{center}
\scriptsize
\setlength{\tabcolsep}{3pt}
\renewcommand{\arraystretch}{1.18}
\begin{tabularx}{\textwidth}{p{0.20\textwidth}p{0.25\textwidth}p{0.23\textwidth}X}
\toprule
Class & Denominator / normalization & Main analytical target & Relation to this paper \\
\midrule
TULA-type Langevin taming & Deterministic drift, state, or full-gradient normalization & Stability and sampling error for deterministic superlinear drift & Background stabilization \\
TUSLA-type stochastic-gradient taming & Deterministic or state-dependent taming of stochastic-gradient updates & Finite-time sampling, optimization, or excess-risk control in nonconvex learning & Related tamed SGLD theory \\
Gradient-dependent denominator & Normalization by the sampled gradient magnitude, e.g. \(1+\alpha_n\|\nabla f(\xi_n,w_n)\|\) & Stabilized stochastic gradient descent or stochastic approximation & Oracle-coupled baseline \\
Deterministic global envelope & \(1+\eta^\alpha A(x)\) or a prescribed deterministic envelope & Oracle-preserving stabilization with a deterministic drift modification & Conservative deterministic baseline \\
Local/hybrid deterministic envelope & Inactive safe region, soft intermediate region, hard-tail safeguard & Stable tamed SGLD with localized deterministic residual & Proposed localized design \\
\bottomrule
\end{tabularx}
\caption{Technical comparison with related tamed stochastic-gradient schemes.}
\label{tab:comparison}
\end{center}

\end{table}

\paragraph{Organization.}
Section~\ref{sec:oracle-transform} studies denominator-induced oracle bias.  Section~\ref{sec:model-recursions} defines the Langevin model and the deterministic-envelope family.  Sections~\ref{sec:moment-stability}--\ref{sec:weak-stationary-identification} prove stability, invariant-measure existence, and small-stepsize stationary consistency.  Section~\ref{sec:quantitative-residual} gives the exact-gradient residual bound.  Section~\ref{sec:sg-extension} extends it to stochastic-gradient kernels.  Section~\ref{sec:experiments} gives numerical tests, and Section~\ref{sec:discussion} discusses proxy construction and weak-error directions.

\section{Stochastic-Gradient Transformations and Denominator-Induced Bias}\label{sec:oracle-transform}

In this section we view taming denominators as transformations of the stochastic-gradient output.  For an oracle output \(v\), a deterministic-envelope transformation is
\[
  \mathcal T_\eta^{\rm det}(x,v)
  =\frac{v}{1+\eta^\alpha A_\eta(x)}.
\]
A gradient-dependent transformation is instead
\[
  \mathcal T_\lambda^{\rm gd}(v)
  =\frac{v}{1+\lambda\|v\|}.
\]
The distinction is structural: conditional on \(x\), \(1+\eta^\alpha A_\eta(x)\) is fixed, whereas \(1+\lambda\|v\|\) remains coupled to the current oracle realization \(v\).

Let \(g(x,U)\) be a stochastic-gradient oracle satisfying
\begin{equation}\label{eq:oracle-unbiased-intro}
  \E[g(x,U)\mid x]=\nabla F_z(x).
\end{equation}
For a deterministic envelope \(A_\eta(x)\), define the scaled oracle output
\[
  G_\eta^{\rm det}(x,U)=\frac{g(x,U)}{1+\eta^\alpha A_\eta(x)}.
\]
If \eqref{eq:oracle-unbiased-intro} holds and \(A_\eta(x)\) depends only on \(x\), then the elementary conditional-mean identity
\begin{equation}\label{eq:oracle-preservation}
  \E[G_\eta^{\rm det}(x,U)\mid x]
  =\frac{\nabla F_z(x)}{1+\eta^\alpha A_\eta(x)}
\end{equation}
holds.  The denominator damps the drift, but it does not add a new
oracle-coupled mean shift.

This identity can fail for a gradient-dependent denominator.  The denominator
then changes the Markov kernel through the current oracle draw, and this can
change the invariant law at fixed stepsize.  The next results quantify this
mean shift.

\begin{proposition}[One-dimensional mean shift from a gradient-dependent denominator]\label{prop:gradient-dependent-denominator-bias}
Let \(g>0\) be a one-dimensional positive random variable with \(\E g=\mu\) and \(\E g^3<\infty\).  For \(\lambda\downarrow0\),
\begin{equation}\label{eq:variance-induced-bias}
  \E\left[\frac{g}{1+\lambda g}\right]
  -\frac{\mu}{1+\lambda\mu}
  =-\lambda\operatorname{Var}(g)+O(\lambda^2).
\end{equation}
In particular, gradient-dependent denominator taming introduces a first-order drift distortion proportional to the oracle variance.
\end{proposition}

\begin{proof}
Use \((1+\lambda g)^{-1}=1-\lambda g+\lambda^2g^2+O(\lambda^3g^3)\) and the corresponding expansion for \((1+\lambda\mu)^{-1}\).  Then
\[
  \E\left[\frac{g}{1+\lambda g}\right]
  =\mu-\lambda\E[g^2]+O(\lambda^2),
  \qquad
  \frac{\mu}{1+\lambda\mu}
  =\mu-\lambda\mu^2+O(\lambda^2),
\]
which gives \eqref{eq:variance-induced-bias}.

\end{proof}

\begin{remark}[Absolute-value denominators]
Proposition~\ref{prop:gradient-dependent-denominator-bias} uses the scalar calculation with \(1+\lambda g\) to display the variance mechanism in its cleanest form.  The same oracle-coupling effect is present for the absolute-value normalization commonly used in scalar tamed stochastic-gradient updates.  Indeed, if \(G\) is a scalar oracle with \(\E|G|^3<\infty\) and \(\mu=\E G\), then, as \(\lambda\downarrow0\),
\[
  \E\left[\frac{G}{1+\lambda |G|}\right]
  -
  \frac{\mu}{1+\lambda |\mu|}
  =
  -\lambda\left(\E[G|G|]-\mu|\mu|\right)
  +O\left(\lambda^2\E|G|^3\right).
\]
When \(G\) has a fixed positive sign, this reduces to the variance term \(-\lambda\operatorname{Var}(G)+O(\lambda^2)\).  If \(G\) can cross zero, the nondifferentiability of \(|\cdot|\) changes the first-order coefficient.  It does not restore the conditional-mean identity. 
\end{remark}

\begin{example}[A two-point oracle]
Let \(G\) be a scalar unbiased oracle with
\[
  G=
  \begin{cases}
  0, & \text{with probability }1/2,\\
  2\mu, & \text{with probability }1/2.
  \end{cases}
\]
Then \(\E G=\mu\), but
\[
  \E\left[\frac{G}{1+\lambda G}\right]
  =\frac{\mu}{1+2\lambda\mu}
  \ne
  \frac{\mu}{1+\lambda\mu}
  =
  \frac{\E G}{1+\lambda\E G}.
\]
Thus, even this elementary unbiased oracle is not preserved by a gradient-dependent denominator.
\end{example}
In higher dimensions, an oracle-dependent denominator can change not only the size but also the direction of the averaged drift.
\begin{proposition}[High-dimensional mean shift from a gradient-dependent denominator]\label{prop:vector-gradient-dependent-denominator-bias}
Let \(G\in\R^d\) be an unbiased vector oracle with \(\E G=\mu\ne0\) and \(\E\|G\|^3<\infty\).  Define
\[
  T_\lambda(v)=\frac{v}{1+\lambda\|v\|},\qquad v\in\R^d .
\]
Then, as \(\lambda\downarrow0\),
\begin{equation}\label{eq:vector-first-order-bias}
  \E[T_\lambda(G)]-T_\lambda(\mu)
  =-\lambda\bigl(\E[G\|G\|]-\mu\|\mu\|\bigr)
   +O\bigl(\lambda^2\E\|G\|^3\bigr).
\end{equation}
Moreover, if \(G=\mu+\varepsilon\xi\), where \(\E\xi=0\), \(\E\|\xi\|^3<\infty\), and \(\Sigma=\E[\xi\xi^\top]\), then
\begin{equation}\label{eq:small-noise-directional-bias}
 \E[G\|G\|]-\mu\|\mu\|
 =\varepsilon^2\left[
   \frac{\Sigma\mu}{\|\mu\|}
   +\frac{\mu}{2\|\mu\|}
   \left(\operatorname{tr}\Sigma-
   \frac{\mu^\top\Sigma\mu}{\|\mu\|^2}\right)
 \right]+O(\varepsilon^3).
\end{equation}
The component \(\Sigma\mu/\|\mu\|\) is parallel to \(\mu\) only when \(\mu\) is an eigenvector of \(\Sigma\).  Thus a directional rotation appears when the oracle covariance is not aligned with the mean-gradient direction; in isotropic or aligned cases the leading bias is only a scalar shrinkage.
\end{proposition}

\begin{proof}
The expansion
\[
  \frac{v}{1+\lambda\|v\|}=v-\lambda v\|v\|+O(\lambda^2\|v\|^3)
\]
with \(v=G\) and \(v=\mu\) gives \eqref{eq:vector-first-order-bias}.  For the small-noise expansion, write \(r=\|\mu\|\) and \(u=\mu/r\).  Taylor expansion of the norm gives
\[
  \|\mu+\varepsilon\xi\|
  =r+\varepsilon u^\top\xi+
  \frac{\varepsilon^2}{2r}\bigl(\|\xi\|^2-(u^\top\xi)^2\bigr)+O(\varepsilon^3\|\xi\|^3).
\]
Multiplying by \(\mu+\varepsilon\xi\), taking expectations, and using \(\E\xi=0\) yields
\[
  \E[(\mu+\varepsilon\xi)\|\mu+\varepsilon\xi\|]-\mu r
  =\varepsilon^2\left[\Sigma u+
  \frac{\mu}{2r}\bigl(\operatorname{tr}\Sigma-u^\top\Sigma u\bigr)\right]+O(\varepsilon^3),
\]
which is \eqref{eq:small-noise-directional-bias}.
\end{proof}


\section{Model, Assumptions, and Deterministic-Envelope Recursions}\label{sec:model-recursions}

We now turn the oracle-level mechanism into the concrete tamed Langevin kernels used in the rest of the paper.  This section fixes the deterministic setting: the drift assumptions, the envelope conditions, and the exact-gradient fixed-stepsize recursion.  The stochastic-gradient kernel is introduced later in  Section~\ref{sec:sg-extension}, after the exact-gradient stationary structure has been established.

\subsection{Langevin model and structural assumptions}

The dataset \(z\) is fixed and
\[
  b(x)=b_z(x)=-\nabla F_z(x).
\]
The Langevin generator is
\begin{equation}\label{eq:generator}
  \Lgen f(x)=\langle b(x),\nabla f(x)\rangle+\beta^{-1}\Delta f(x).
\end{equation}
Throughout the paper, constants may depend on fixed structural parameters, such as \(d,\beta,p\), polynomial-growth orders, and Lyapunov constants, but not on the iteration index or on sufficiently small \(\eta\), unless explicitly stated.

\begin{assumption}[Polynomial local Lipschitz drift]\label{assump:poly-lip}
There exist constants \(L>0\) and \(r>0\) such that for all \(x,y\in\R^d\),
\begin{equation}\label{eq:poly-lip}
 \|b(x)-b(y)\|\le L(1+\|x\|^r+\|y\|^r)\|x-y\|.
\end{equation}
Equivalently, if \(b\in C^1\), 
\begin{equation}\label{eq:grad-b-growth}
 \|\nabla b(x)\|_{\op}\le L(1+\|x\|^r).
\end{equation}
\end{assumption}

\begin{assumption}[Dissipativity]\label{assump:dissipative}
There exist constants \(m>0\) and \(b_0\ge0\) such that
\begin{equation}\label{eq:dissipative}
  \langle x,b(x)\rangle\le b_0-m\|x\|^{r+2},\qquad x\in\R^d.
\end{equation}
\end{assumption}

\begin{assumption}[Taming envelope]\label{assump:envelope}
There is a deterministic continuous envelope \(A:\R^d\to[0,\infty)\) and constants \(C_A,L_A>0\) such that
\begin{equation}\label{eq:envelope-dom}
  \|\nabla b(x)\|_{\op}\le A(x)\le L_A(1+\|x\|^r),
\end{equation}
\begin{equation}\label{eq:taming-compatible}
  \frac{\|b(x)\|}{1+A(x)}\le C_A(1+\|x\|),\qquad x\in\R^d.
\end{equation}
\end{assumption}

The assumptions have separate roles.  Assumptions~\ref{assump:poly-lip}--\ref{assump:dissipative} describe a polynomial-growth dissipative drift.  Assumption~\ref{assump:envelope} describes an envelope that controls the far-field size of the explicit drift increment.  The compatibility bound \eqref{eq:taming-compatible} is the key point; it will be used below to make the tamed increment effectively linear.

Here, deterministic means that the denominator is fixed by the current state and the dataset, and is not recomputed from the fresh stochastic-gradient sample used in one update. This distinction is used again in the stochastic-gradient extension in Section~\ref{sec:sg-extension}.

For many common empirical risks with polynomial growth, the preceding assumptions can be checked directly. The key point is Assumption~\ref{assump:envelope}. The following quartic example gives a simple choice of \(A\) satisfying it.

\begin{example}[Quartic empirical risk]\label{ex:quartic-risk}
Consider the regularized quartic empirical risk
\begin{equation}\label{eq:quartic-risk}
  F_z(w)=\frac1N\sum_{i=1}^N \frac14(a_i^\top w-y_i)^4+\frac\lambda2\|w\|^2,
  \qquad \lambda>0.
\end{equation}
Assume that the data are bounded and that the quartic loss is coercive in all directions, for instance under a standard nondegeneracy condition on the covariates.  Then \(b(w)=-\nabla F_z(w)\) has cubic growth, while \(\nabla b(w)\) has quadratic growth.  Hence the polynomial local-Lipschitz condition holds with \(r=2\), and the coercivity gives the order-four dissipativity required in Assumption~\ref{assump:dissipative}.  Moreover, a Hessian-type deterministic envelope
\[
  A(w)=C_z(1+\|w\|^2)
\]
satisfies
\[
  \|\nabla b(w)\|_{\op}\le A(w),
  \qquad
  \frac{\|b(w)\|}{1+A(w)}\le C(1+\|w\|).
\]
Thus this quartic empirical-risk class gives a non-globally Lipschitz example covered by the structural assumptions.
\end{example}

\subsection{Localized deterministic envelope family}\label{subsec:LDE family}

We next define a family of localized deterministic envelopes.  

\begin{definition}[Localized deterministic-envelope family]\label{def:hybrid-envelope}
Let \(\bar A:\R^d\to[0,\infty)\) be a deterministic growth indicator.  For parameters \(R\ge0\), \(0<\theta\le1\), a threshold \(S_\eta\ge R\), and activation indicators \(c_s,c_h\in\{0,1\}\), define
\begin{equation}\label{eq:hybrid-envelope}
  A_\eta^{c_s,c_h}(x)
  \coloneqq c_s\big(\bar A(x)-R\big)_+^\theta+c_h\big(\bar A(x)-S_\eta\big)_+.
\end{equation}
\end{definition}

The threshold \(R\) defines the safe region
\[
  \mathcal S_R \coloneqq \{x:\bar A(x)\le R\}.
\]
If \(S_\eta>R\), the intermediate region is
\[
  \mathcal I_{R,S_\eta} \coloneqq \{x:R<\bar A(x)\le S_\eta\}.
\]
The far-tail region is
\[
  \mathcal T_{S_\eta} \coloneqq \{x:\bar A(x)>S_\eta\}.
\]
We call these envelopes localized because they are inactive on the safe region and become active only after the prescribed growth thresholds are crossed.  The exponent \(\theta\) defines the soft component: it lets taming turn on gradually after the safe region, so that the intermediate region can be controlled without unnecessary drift modification. The second term is therefore the hard-tail safeguard, active only in \(\mathcal T_{S_\eta}\).  The hybrid envelope combines soft intermediate control with hard-tail protection.  The local soft, local hard, and hybrid envelopes are
\[
  A_\eta^{\rm soft} \coloneqq A_\eta^{1,0},\qquad
  A_\eta^{\rm hard} \coloneqq A_\eta^{0,1},\qquad
  A_\eta^{\rm hyb} \coloneqq A_\eta^{1,1}.
\]
When no distinction among these choices is needed, we write \(A_\eta\) for a generic deterministic envelope.

The localized family is based on the idea that different regions require different amounts of taming. In the safe region, \(A_\eta=0\), so the drift is used without taming. In the intermediate region, the drift may already be large enough to need a mild correction; the soft component plays this role, with \(\theta\) controlling its strength. In the far tail, stronger taming is needed, and the hard component provides this safeguard. The thresholds \(R\) and \(S_\eta\) determine these regions, and the hybrid rule combines the soft mild correction with the hard-tail safeguard and serves as a robust default, but the best localized envelope may depend on the problem.

\subsection{Envelope compatibility of localized rules}

The later stability theory uses the effective-linearity consequence of Assumption~\ref{assump:envelope}.  Global envelopes satisfy this condition by assumption.  Localized rules require a separate check.  The check below assumes that the growth indicator itself has the strong stabilization needed in the far tail.  It then shows that rules with a hard-tail safeguard, such as local hard and hybrid envelopes, retain the needed control.  Local soft rules do not satisfy this bound automatically.  The soft component alone may not provide enough far-tail stabilization, which motivates the hard-tail safeguard.

\begin{lemma}[Effective-linearity bound for localized envelopes with a hard-tail safeguard]\label{lem:hybrid-step-control}
Let \(S_\eta=c_S\eta^{-\alpha}\).  Let \(A_\eta^{c_s,1}\) be a localized envelope from \eqref{eq:hybrid-envelope} with a hard-tail safeguard.  Assume that the deterministic growth indicator satisfies
\begin{equation}\label{eq:localized-growth-compatibility}
  \frac{\|b(x)\|}{1+\bar A(x)}\le C_b(1+\|x\|),
  \qquad
  \bar A(x)\le C_b(1+\|x\|^r).
\end{equation}
Then there is a constant \(C\), independent of sufficiently small \(\eta\), such that
\begin{equation}\label{eq:localized-step-control}
  \eta\frac{\|b(x)\|}{1+\eta^\alpha A_\eta^{c_s,1}(x)}
  \le C\eta^{1-\alpha}(1+\|x\|),
  \qquad x\in\R^d .
\end{equation}
Moreover, \(A_\eta^{c_s,1}\) has the same polynomial upper growth as \(\bar A\):
\begin{equation}\label{eq:localized-upper-growth}
  A_\eta^{c_s,1}(x)
  \le C(1+\bar A(x))
  \le C(1+\|x\|^r).
\end{equation}
\end{lemma}

\begin{proof}
We first check that the localized envelope has the required polynomial growth. Since \(0<\theta\le1\), \(R\ge0\), and \(c_s\) is fixed,
\[
A_\eta^{c_s,1}(x)
=c_s(\bar A(x)-R)_+^\theta+(\bar A(x)-S_\eta)_+
\le C(1+\bar A(x)).
\]
Together with \(\bar A(x)\le C(1+\|x\|^r)\), this gives
\[
A_\eta^{c_s,1}(x)\le C(1+\|x\|^r).
\]

On the safe region \(\bar A(x)\le R\), the localized envelope is inactive.  Since \(R\le S_\eta\) for all sufficiently small \(\eta\), \eqref{eq:localized-growth-compatibility} gives
\[
  \eta\frac{\|b(x)\|}{1+\eta^\alpha A_\eta^{c_s,1}(x)}
  =\eta\|b(x)\|
  \le C\eta(1+R)(1+\|x\|)
  \le C\eta(1+S_\eta)(1+\|x\|).
\]
With \(S_\eta=c_S\eta^{-\alpha}\), this is bounded by
\(C\eta^{1-\alpha}(1+\|x\|)\).

On the intermediate region \(R<\bar A(x)\le S_\eta\), \(1+\eta^\alpha A_\eta^{c_s,1}(x)\ge1\).  Hence
\[
  \eta\frac{\|b(x)\|}{1+\eta^\alpha A_\eta^{c_s,1}(x)}
  \le \eta\|b(x)\|
  \le C\eta(1+\bar A(x))(1+\|x\|)
  \le C\eta(1+S_\eta)(1+\|x\|)
  \le C\eta^{1-\alpha}(1+\|x\|).
\]
This is the region where the soft part may be active, but the growth indicator is still below the hard-tail threshold.

Then, for the far-tail region \(\bar A(x)>S_\eta\), we first consider the moderate-tail region \(S_\eta<\bar A(x)\le 2S_\eta\).   Therefore the same argument as in the preceding case, using only the lower bound \(1+\eta^\alpha A_\eta^{c_s,1}(x)\ge1\), gives
\[
  \eta\frac{\|b(x)\|}{1+\eta^\alpha A_\eta^{c_s,1}(x)}
  \le C\eta^{1-\alpha}(1+\|x\|).
\]

Finally, on the extreme-tail region \(\bar A(x)>2S_\eta\), the hard-tail safeguard controls the growth indicator:
\[
  (\bar A(x)-S_\eta)_+\ge \frac12\bar A(x).
\]
Hence
\[
  1+\eta^\alpha A_\eta^{c_s,1}(x)
  \ge 1+\frac12\eta^\alpha\bar A(x).
\]
Using \eqref{eq:localized-growth-compatibility},
\[
  \eta\frac{\|b(x)\|}{1+\eta^\alpha A_\eta^{c_s,1}(x)}
  \le
  2\eta^{1-\alpha}\frac{\|b(x)\|}{1+\bar A(x)}
  \le C\eta^{1-\alpha}(1+\|x\|).
\]
Combining the four regimes proves \eqref{eq:localized-step-control}.  
\end{proof}

The preceding proof also gives a useful diagnostic for local soft rules. The far-tail part of Lemma~\ref{lem:hybrid-step-control} uses the fact that the envelope is comparable with \(\bar A\) when \(\bar A(x)\) is large. A purely soft rule
\[
  A_\eta^{\rm soft}(x)=(\bar A(x)-R)_+^\theta
\]
does not have this property when \(0<\theta<1\), since
\[
  \frac{A_\eta^{\rm soft}(x)}{\bar A(x)}
  \sim \bar A(x)^{\theta-1}\to0
  \qquad\text{as }\bar A(x)\to\infty .
\]
Thus the soft component gives only sublinear far-tail growth. The hard-tail component is added to restore the linear comparability needed in the Lyapunov argument.

\subsection{Exact-gradient deterministic-envelope recursion}\label{subsec:exact-gradient-recursion}

We now define the exact-gradient Markov chain that serves as the base stationary object. 
Fix \(\alpha\in(0,1/2)\).  For a constant stepsize \(\eta\in(0,\eta_0]\), define
\begin{equation}\label{eq:tamed-drift}
  b^T_\eta(x)=\frac{b(x)}{1+\eta^\alpha A_\eta(x)},
\end{equation}
where \(A_\eta\) may be the global envelope \(A\) from Assumption~\ref{assump:envelope} or any localized deterministic envelope satisfying the same effective-linearity condition.  

The exact-gradient tamed Langevin recursion is
\begin{equation}\label{eq:tamed-chain}
  Y_{k+1}=Y_k+\eta b^T_\eta(Y_k)+\sqrt{2\beta^{-1}\eta}\,\xi_{k+1},
\end{equation}
where \((\xi_k)\) are independent standard Gaussian vectors in \(\R^d\).  We denote the Markov kernel of this chain by \(\Qeta\).  Equivalently, if \(X_0=x\) and
\[
  X_1=x+\Delta_\eta(x),\qquad
  \Delta_\eta(x) \coloneqq \eta b^T_\eta(x)+\sqrt{2\beta^{-1}\eta}\,\xi,
\]
then, for every bounded measurable test function \(\varphi\),
\begin{equation}\label{eq:kernel-conditional-expectation}
  \Qeta \varphi(x)
  =\E\big[\varphi(X_1)\mid X_0=x\big]
  =\E\big[\varphi(x+\Delta_\eta(x))\big].
\end{equation}

The exponent \(\alpha\in(0,1/2)\) controls the strength of taming.  Smaller \(\alpha\) gives stronger stabilization, while values closer to \(1/2\) keep the modified drift closer to the original drift.  The restriction \(\alpha<1/2\) is used in the later residual estimates.  In the experiments we use \(\alpha=0.45\) to avoid overly aggressive taming while retaining visible stabilization.

\section{Lyapunov Stability and Invariant Measures}\label{sec:moment-stability}

We now prove Lyapunov stability and invariant-measure existence for the exact-gradient deterministic-envelope kernel \(\Qeta\).  The key input is the effective-linearity bound in Assumption~\ref{assump:envelope}.  For localized envelopes with a hard-tail safeguard, the corresponding bound is verified in Lemma~\ref{lem:hybrid-step-control}.

We first collect two elementary estimates used in the Lyapunov drift argument.  The first one controls the tamed drift under the envelope condition, while the second one converts one-step updates into increments of the polynomial Lyapunov function.

\begin{lemma}[Effective-linearity bound for the tamed drift step]\label{lem:effective-linear}
Under Assumption~\ref{assump:envelope}, for every \(\eta\in(0,1]\),
\begin{equation}\label{eq:effective-linear}
  \eta\|b^T_\eta(x)\|\le C\eta^{1-\alpha}(1+\|x\|),\qquad x\in\R^d.
\end{equation}
\end{lemma}

\begin{proof}
Since \(1+\eta^\alpha A(x)\ge \eta^\alpha(1+A(x))\),
\[
  \eta\|b^T_\eta(x)\|
  =\eta\frac{\|b(x)\|}{1+\eta^\alpha A(x)}
  \le \eta^{1-\alpha}\frac{\|b(x)\|}{1+A(x)}
  \le C\eta^{1-\alpha}(1+\|x\|).
\]
\end{proof}

\begin{lemma}[Global increment bound for polynomial Lyapunov functions]\label{lem:poly-increment}
Let \(q\ge2\).  There exists a constant \(C_q<\infty\) such that, for all \(x,u\in\R^d\),
\begin{equation}\label{eq:poly-taylor}
 V_q(x+u)-V_q(x)
 \le q\|x\|^{q-2}\langle x,u\rangle
 +C_q(1+\|x\|^{q-2})\|u\|^2+C_q\|u\|^q,
 \qquad V_q(x)=1+\|x\|^q.
\end{equation}
\end{lemma}

\begin{proof}
When \(x=0\), the estimate reduces to
\[
  \|u\|^q \le C_q\|u\|^2+C_q\|u\|^q,
\]
which holds for \(C_q\ge1\). We therefore assume \(x\ne0\).

Set \(\phi(y)=\|y\|^q\).  If \(\|u\|\le \|x\|/2\), then the segment \(x+tu\), \(0\le t\le1\), stays away from the origin and
\[
  \|\nabla^2\phi(x+tu)\|\le C_q\|x\|^{q-2}.
\]
Taylor's formula gives
\[
  \phi(x+u)-\phi(x)
  \le \langle \nabla\phi(x),u\rangle+C_q\|x\|^{q-2}\|u\|^2,
\]
and \(\nabla\phi(x)=q\|x\|^{q-2}x\).  

For the large-increment case, suppose \(\|u\|>\|x\|/2\). Then \(\|x\|\le2\|u\|\). By the triangle inequality,
\[
  \|x+u\|^q\le 2^{q-1}(\|x\|^q+\|u\|^q).
\]
Also,
\[
  \|\nabla\phi(x)\|=q\|x\|^{q-1}.
\]
Therefore,
\[
\begin{aligned}
  \phi(x+u)-\phi(x)-\langle\nabla\phi(x),u\rangle
  &\le 2^{q-1}(\|x\|^q+\|u\|^q)
      +q\|x\|^{q-1}\|u\|  \\
  &\le C_q\|u\|^q,
\end{aligned}
\]
where the last step uses \(\|x\|\le2\|u\|\).

Combining the two cases, and adding a nonnegative term \(C_q(1+\|x\|^{q-2})\|u\|^2\)  proves the claim.
\end{proof}

\begin{proposition}[Polynomial Lyapunov stability]\label{prop:moment-stability}
Assume Assumptions~\ref{assump:poly-lip}--\ref{assump:envelope}.  For every \(q\ge2\), there exist constants \(c_q>0\), \(C_q<\infty\), and \(\eta_0>0\), independent of \(x\) and \(\eta\), such that for all \(\eta\in(0,\eta_0]\), the tamed chain \eqref{eq:tamed-chain} satisfies
\begin{equation}\label{eq:uniform-moment-trajectory}
  \sup_{k\ge0}\E[1+\|Y_k\|^q]
  \le C_q(1+\E\|Y_0\|^q).
\end{equation}
Equivalently, with \(V_q(x)=1+\|x\|^q\),
\begin{equation}\label{eq:drift-condition}
  \Qeta V_q(x)\le (1-c_q\eta)V_q(x)+C_q\eta.
\end{equation}
\end{proposition}

\begin{proof}
Throughout this proof, \(C_q\) denotes a finite positive constant depending only on \(q\) and the model constants; its value may increase from line to line.

First, we rewrite the kernel action in one-step form.  By the conditional-expectation representation \eqref{eq:kernel-conditional-expectation}, applied to the Lyapunov function \(V_q\),
\[
  \Qeta V_q(x)
  =\E\big[V_q(X_1)\mid X_0=x\big]
  =\E\big[V_q(x+\Delta_\eta(x))\big],
\]
where
\[
  \Delta_\eta(x)=\eta b^T_\eta(x)+\sqrt{2\beta^{-1}\eta}\,\xi,
  \qquad \xi\sim N(0,I_d).
\]
Applying Lemma~\ref{lem:poly-increment} with \(u=\Delta_\eta(x)\) and taking expectation gives the decomposition
\begin{align}
  \Qeta V_q(x)-V_q(x)
  &\le q\eta\|x\|^{q-2}\langle x,b^T_\eta(x)\rangle  \notag\\
  &\quad +C_q(1+\|x\|^{q-2})\E\|\Delta_\eta(x)\|^2
       +C_q\E\|\Delta_\eta(x)\|^q .
  \label{eq:lyapunov-basic-decomposition}
\end{align}
Here the Gaussian part does not contribute to the first-order term because \(\E\xi=0\).

We next estimate the first-order drift term.  By Assumption~\ref{assump:dissipative},
\[
  \langle x,b^T_\eta(x)\rangle
  =\frac{\langle x,b(x)\rangle}{1+\eta^\alpha A(x)}
  \le \frac{b_0}{1+\eta^\alpha A(x)}
       -m\frac{\|x\|^{r+2}}{1+\eta^\alpha A(x)}.
\]
The positive part is bounded by \(b_0\).  For the negative part, Assumption~\ref{assump:envelope} and \(0<\eta\le1\) give
\[
  1+\eta^\alpha A(x)
  \le 1+L_A(1+\|x\|^r)
  \le C_A(1+\|x\|^r).
\]
There exists some constants \(C_r, c_r>0\) such that the denominator reduces the order from \(r+2\) to 2, but does not destroy dissipation:
\begin{equation}\label{eq:key-ratio}
  \frac{\|x\|^{r+2}}{1+\eta^\alpha A(x)}
  \ge C^{-1}_A\frac{\|x\|^{r+2}}{1+\|x\|^r}
  \ge c_r\|x\|^2-C_r .
\end{equation}
Consequently, let \(C= b_0+mC_r, c=mc_r\)
\[
  \langle x,b^T_\eta(x)\rangle
  \le C-c\|x\|^2 .
\]
Multiplying by \(q\eta\|x\|^{q-2}\),
\[
  q\eta\|x\|^{q-2}\langle x,b^T_\eta(x)\rangle
  \le Cq\eta\|x\|^{q-2}-cq\eta\|x\|^q .
\]
By Young's inequality, for any \(\varepsilon>0\) there exists
\(C_{\varepsilon,q}<\infty\) such that
\[
  \|x\|^{q-2}
  \le \varepsilon \|x\|^q+C_{\varepsilon,q}.
\]
Choosing \(\varepsilon>0\) small enough so that
\(qC\varepsilon\le qc/2\), we obtain
\[
  qC\eta\|x\|^{q-2}
  \le \frac{qc}{2}\eta\|x\|^q+C_q\eta .
\]
Therefore,
\[
  q\eta\|x\|^{q-2}\langle x,b^T_\eta(x)\rangle
  \le -\frac{qc}{2}\eta\|x\|^q+C_q\eta .
\]
Set \(\lambda_q=qc/2\).  Then
\begin{equation}\label{eq:dissipative-tamed}
  q\eta\|x\|^{q-2}\langle x,b^T_\eta(x)\rangle
  \le -\lambda_q\eta\|x\|^q+C_q\eta .
\end{equation}
It remains to control the quadratic and higher-order remainders in \eqref{eq:lyapunov-basic-decomposition}. By Lemma~\ref{lem:effective-linear}, the elementary \(\|a+b\|^m\le c_m(\|a\|^m+\|b\|^m)\), and \(\E\|\xi\|^m<\infty\), for every \(m\ge2\), there exists a constant \(C_m>0\) such that
\begin{equation}\label{eq:delta-moment-bound}
  \E\|\Delta_\eta(x)\|^m
  \le C_m\eta^{m(1-\alpha)}(1+\|x\|^m)+C_m\eta^{m/2} .
\end{equation}

Taking \(m=2\) in \eqref{eq:delta-moment-bound} gives
\[
  \E\|\Delta_\eta(x)\|^2
  \le C_2\eta^{2-2\alpha}(1+\|x\|^2)+C_2\eta .
\]
Moreover,
\[
  (1+\|x\|^{q-2})(1+\|x\|^2)
  \le K_q(1+\|x\|^q).
\]
Hence
\[
\begin{aligned}
  (1+\|x\|^{q-2})\E\|\Delta_\eta(x)\|^2
  &\le C_2K_q\eta^{2-2\alpha}(1+\|x\|^q)
      +C_2\eta(1+\|x\|^{q-2}) .
\end{aligned}
\]
By Young's inequality, for every \(\varepsilon>0\) there exists
\(C_{\varepsilon,q}<\infty\) such that
\[
  1+\|x\|^{q-2}
  \le \varepsilon\|x\|^q+C_{\varepsilon,q}.
\]
Choosing \(\varepsilon>0\) so small that
\(C_2\varepsilon\le \lambda_q/8\), we obtain
\[
  C_2\eta(1+\|x\|^{q-2})
  \le \frac{\lambda_q}{8}\eta\|x\|^q+C_q\eta .
\]
After increasing the generic constant \(C_q\), it follows that
\begin{equation}\label{eq:quadratic-remainder-bound}
  (1+\|x\|^{q-2})\E\|\Delta_\eta(x)\|^2
  \le C_q\eta^{2-2\alpha}(1+\|x\|^q)
      +\frac{\lambda_q}{8}\eta\|x\|^q+C_q\eta .
\end{equation}

Taking \(m=q\) in \eqref{eq:delta-moment-bound} gives
\begin{equation}\label{eq:q-remainder-bound}
  \E\|\Delta_\eta(x)\|^q
  \le C_q\eta^{q(1-\alpha)}(1+\|x\|^q)+C_q\eta^{q/2}.
\end{equation}
Since \(\alpha<1/2\) and \(q\ge2\), we have
\[
  2-2\alpha>1,
  \qquad
  q(1-\alpha)>1 .
\]
Thus, after decreasing \(\eta_0>0\) if necessary,
\[
  C_q\eta^{2-2\alpha}\|x\|^q
  +C_q\eta^{q(1-\alpha)}\|x\|^q
  \le \frac{\lambda_q}{8}\eta\|x\|^q,
  \qquad 0<\eta\le\eta_0 .
\]

Combining \eqref{eq:dissipative-tamed},
\eqref{eq:quadratic-remainder-bound}, and
\eqref{eq:q-remainder-bound} with
\eqref{eq:lyapunov-basic-decomposition}, we obtain, for
\(0<\eta\le\eta_0\),
\[
  \Qeta V_q(x)-V_q(x)
  \le -\frac{3\lambda_q}{4}\eta\|x\|^q+C_q\eta .
\]
After renaming \(3\lambda_q/4\) as \(c_q\) and increasing \(C_q\) if necessary,
\[
  \Qeta V_q(x)-V_q(x)
  \le -c_q\eta\|x\|^q+C_q\eta .
\]
Since \(V_q(x)=1+\|x\|^q\), this implies
\[
  \Qeta V_q(x)
  \le (1-c_q\eta)V_q(x)+C_q\eta .
\]
\end{proof}

\begin{theorem}[Invariant measures of the exact-gradient deterministic-envelope kernel]\label{thm:invariant-measure}
Under the assumptions of Proposition~\ref{prop:moment-stability}, there exist constants \(C_q<\infty\) and \(\eta_0>0\) such that, for every \(0<\eta\le\eta_0\), the exact-gradient deterministic-envelope kernel \(\Qeta\) admits at least one invariant probability measure.  Moreover, every invariant probability measure \(\pi_{\eta}\) of \(\Qeta\) satisfies
\begin{equation}\label{eq:inv-moment-vq}
  \int V_q(x)\,\pi_{\eta}(\dd x)\le C_q .
\end{equation}
Equivalently, for every \(q\ge2\) covered by Proposition~\ref{prop:moment-stability},
\begin{equation}\label{eq:inv-moment}
  \sup_{0<\eta\le\eta_0}\int \|x\|^q\,\pi_{\eta}(\dd x)<\infty .
\end{equation}
\end{theorem}

\begin{proof}
The kernel \(\Qeta\) is Feller because the one-step update map is continuous in the initial state and the Gaussian noise is fixed.  Fix \(x\in\mathbb R^d\) and define the Krylov--Bogoliubov averages
\[
  \mu_n^{\eta,x}:=\frac1n\sum_{k=0}^{n-1}\delta_x Q_\eta^k .
\]
The drift estimate \eqref{eq:drift-condition} and the trajectory moment bound \eqref{eq:uniform-moment-trajectory} imply
\[
  \sup_{n\ge1}\int V_q(y)\,\mu_n^{\eta,x}(\dd y)<\infty .
\]
Since \(V_q(y)\to\infty\) as \(\|y\|\to\infty\), the family \((\mu_n^{\eta,x})_{n\ge1}\) is tight.  By the standard Krylov--Bogoliubov argument for Feller kernels, any weak limit point is an invariant probability measure of \(\Qeta\).

Let \(\pi_{\eta}\) be any invariant probability measure of \(\Qeta\).  Integrating \eqref{eq:drift-condition} with respect to \(\pi_{\eta}\) gives
\[
  \int V_q\,\dd\pi_{\eta}
  =
  \int \Qeta V_q\,\dd\pi_{\eta}
  \le
  (1-c_q\eta)\int V_q\,\dd\pi_{\eta}+C_q\eta .
\]
Hence
\[
  \int V_q\,\dd\pi_{\eta}\le \frac{C_q}{c_q} .
\]
After increasing \(C_q\) if necessary, this yields \eqref{eq:inv-moment-vq} and therefore \eqref{eq:inv-moment}.
\end{proof}

Theorem~\ref{thm:invariant-measure} shows that invariant measures of \(Q_\eta\) satisfy a uniform Lyapunov bound.  To avoid repeating this condition, we use the following terminology.

\begin{definition}[Admissible families of invariant measures]\label{def:admissible}
Let \((K_\eta)_{0<\eta\le\eta_0}\) be a family of Markov kernels on \(\mathbb R^d\), and fix a moment order \(q\ge2\).  A family \((\pi_\eta)_{0<\eta\le\eta_0}\) is called an \emph{admissible family of invariant measures for \(K_\eta\) at order \(q\)} if, for every \(0<\eta\le\eta_0\),
\[
  \pi_\eta K_\eta=\pi_\eta,
\]
and the selected invariant measures satisfy the uniform Lyapunov-moment bound
\[
  \sup_{0<\eta\le\eta_0}\int_{\mathbb R^d} V_q(x)\,\pi_\eta(\dd x)<\infty .
\]
When several moment orders are needed, admissibility means admissibility at all those orders.
\end{definition}
For the exact-gradient deterministic-envelope kernel \(Q_\eta\),
Theorem~\ref{thm:invariant-measure} shows that any selection of invariant
probability measures is admissible at every polynomial moment order covered by
the theorem.  The uniform \(V_q\)-moment bound will be used below as the
tightness input: since \(V_q(x)\to\infty\) as \(\|x\|\to\infty\), every sequence
\(\eta_j\downarrow0\) has a subsequence along which
\(\pi_{\eta_j}\Rightarrow\pi\).

We also need the same Lyapunov input for the localized exact-gradient kernel with a hard-tail safeguard.  The next proposition records this extension.

\begin{proposition}[Drift estimate for a localized exact-gradient kernel with hard-tail safeguard]\label{prop:hybrid-moment-stability}
Assume Assumptions~\ref{assump:poly-lip} and \ref{assump:dissipative}.  Suppose the computable envelope \(\bar A\) satisfies \eqref{eq:localized-growth-compatibility}, and let \(A_\eta^{c_s,1}\) be a localized envelope from \eqref{eq:hybrid-envelope} with threshold \(S_\eta=c_S\eta^{-\alpha}\).  Then the corresponding localized exact-gradient kernel \(Q_\eta^{c_s,1}\) satisfies the following drift estimate: for every \(q\ge2\) covered by the polynomial growth assumptions, there are constants \(c_q>0\), \(C_q<\infty\), and \(\eta_0>0\) such that
\begin{equation}\label{eq:hybrid-drift-condition}
  Q_\eta^{c_s,1}V_q(x)
  \le (1-c_q\eta)V_q(x)+C_q\eta,
  \qquad 0<\eta\le\eta_0 .
\end{equation}
Consequently, the localized exact-gradient kernel \(Q_\eta^{c_s,1}\) admits invariant probability measures, and every such invariant measure satisfies the same uniform \(V_q\)-moment bound as in Theorem~\ref{thm:invariant-measure}.
\end{proposition}

\begin{proof}
The proof repeats the proof of Proposition~\ref{prop:moment-stability},
with the envelope inputs supplied by Lemma~\ref{lem:hybrid-step-control}.
For the localized envelope \(A_\eta^{c_s,1}\), write
\[
  b_\eta^{c_s,1}(x)
  :=
  \frac{b(x)}{1+\eta^\alpha A_\eta^{c_s,1}(x)} .
\]
Lemma~\ref{lem:hybrid-step-control} gives the effective-linearity bound
\[
  \|b_\eta^{c_s,1}(x)\|
  \le C\eta^{-\alpha}(1+\|x\|)
\]
and the upper-growth bound needed in the dissipativity estimate.  In
particular, the same argument as in \eqref{eq:key-ratio} yields
\[
  \frac{\|x\|^{r+2}}{1+\eta^\alpha A_\eta^{c_s,1}(x)}
  \ge c_r\|x\|^2-C_r .
\]
Using Assumption~\ref{assump:dissipative}, we obtain
\[
  \langle x,b_\eta^{c_s,1}(x)\rangle
  \le C-c\|x\|^2 .
\]
Therefore, by Young's inequality, there exists \(\lambda_q>0\) such that
\[
  q\eta\|x\|^{q-2}\langle x,b_\eta^{c_s,1}(x)\rangle
  \le -\lambda_q\eta\|x\|^q+C_q\eta .
\]

The increment estimate of Lemma~\ref{lem:poly-increment} applies to
\[
  \Delta_\eta^{c_s,1}(x)
  =
  \eta b_\eta^{c_s,1}(x)+\sqrt{2\eta}\xi .
\]
The effective-linearity bound above gives, for every \(m\ge2\),
\[
  \E\|\Delta_\eta^{c_s,1}(x)\|^m
  \le
  C_m\eta^{m(1-\alpha)}(1+\|x\|^m)+C_m\eta^{m/2}.
\]
Thus the quadratic and \(q\)-th order remainders are the same as in
Proposition~\ref{prop:moment-stability}.  Since \(\alpha<1/2\), they are
absorbed by the negative drift after decreasing \(\eta_0\) if necessary.
This gives
\[
  Q_\eta^{c_s,1}V_q(x)
  \le (1-c_q\eta)V_q(x)+C_q\eta .
\]

The existence of invariant probability measures and the uniform
\(V_q\)-moment bound follow from the same Krylov--Bogoliubov and stationary
moment argument used in Theorem~\ref{thm:invariant-measure}, with the above
drift condition replacing \eqref{eq:drift-condition}.

\end{proof}

\section{Stationary Consistency and Polynomial Observables}\label{sec:weak-stationary-identification}

Section~\ref{sec:moment-stability} gives invariant measures for each fixed-stepsize exact-gradient deterministic-envelope kernel \(\Qeta\). Since \(\Qeta\) is an Euler-type numerical kernel with a modified drift, its invariant measures need not equal the Gibbs law \(\piz\) at fixed \(\eta\). This section proves that every admissible family \((\pi_\eta)_{0<\eta\le\eta_0}\) for \(\Qeta\) satisfies
\[
  \pi_\eta\Rightarrow\piz,\qquad \eta\downarrow0,
\]
and then extends this convergence to polynomial-growth observables.

Admissibility makes \((\pi_\eta)_{0<\eta\le\eta_0}\) tight, so subsequential weak limits exist as \(\eta\downarrow0\). To identify these limits with the Gibbs law, we impose the following stationary identification condition for the limiting Langevin generator.

\begin{assumption}[Stationary identification of the Gibbs law]\label{assump:gibbs-unique}
Let
\[
  \Lgen f(x)
  := \langle b(x),\nabla f(x)\rangle+\beta^{-1}\Delta f(x),
  \qquad f\in C_c^\infty(\R^d).
\]
If a probability measure \(\mu\) on \(\R^d\) satisfies
\[
  \int \Lgen f\,\dd\mu=0,
  \qquad f\in C_c^\infty(\R^d),
\]
then \(\mu=\piz\).
\end{assumption}
\begin{remark}[Role of the identification assumption]
Assumption~\ref{assump:gibbs-unique} is used only to identify subsequential
stationary limits of the numerical kernels.  Once a limit \(\pi\) satisfies
\[
  \int \Lgen f\,\dd\pi=0,
  \qquad f\in C_c^\infty(\R^d),
\]
the assumption gives \(\pi=\piz\).  It is not a uniqueness assumption for the
fixed-stepsize kernels.  In standard overdamped Langevin settings, it follows
from the usual dissipativity, nondegenerate-noise, and recurrence/irreducibility
conditions; see \cite{MeynTweedie1993,MattinglyStuartHigham2002}.
\end{remark}


\begin{lemma}[Local generator consistency for compactly supported test functions]\label{lem:compact-generator-consistency}
Assume Assumptions~\ref{assump:poly-lip}--\ref{assump:envelope}.  Let \(f\in C_c^\infty(\R^d)\).  Then, for every compact set \(K\subset\R^d\),
\begin{equation}\label{eq:compact-generator-consistency}
  \sup_{x\in K}\left|
  \frac{\Qeta f(x)-f(x)}{\eta}-\Lgen f(x)
  \right|\longrightarrow 0,
  \qquad \eta\downarrow0 .
\end{equation}
More generally, \eqref{eq:compact-generator-consistency} holds with
\(\Qeta\) replaced by \(Q_\eta^A\), where \(Q_\eta^A\) is the exact-gradient
kernel with tamed drift
\[
  b_\eta^A(x)
  :=
  \frac{b(x)}{1+\eta^\alpha A_\eta(x)} ,
\]
provided that \(A_\eta\) is deterministic 
and satisfies, for every compact \(K\subset\R^d\),
\[
  \sup_{x\in K}\eta^\alpha A_\eta(x)\to0,
  \qquad \eta\downarrow0 .
\]
\end{lemma}

\begin{proof}
We prove the statement for \(Q_\eta^A\); the case \(Q_\eta\) is obtained by
taking \(A_\eta=A\).  Fix a compact set \(K\subset\R^d\).  By
Assumption~\ref{assump:poly-lip}, \(b\) is bounded on \(K\).  Since
\[
  b_\eta^A(x)=\frac{b(x)}{1+\eta^\alpha A_\eta(x)},
\]
we have
\[
  \sup_{x\in K}\|b_\eta^A(x)-b(x)\|
  \le
  \Big(\sup_{x\in K}\|b(x)\|\Big)
  \Big(\sup_{x\in K}\eta^\alpha A_\eta(x)\Big)
  \to0 .
\]
Thus \(b_\eta^A\to b\) uniformly on \(K\), and \(b_\eta^A\) is bounded on \(K\)
for all sufficiently small \(\eta\).

Writing one step as
\[
  x+\eta b_\eta^A(x)+\sqrt{2\beta^{-1}\eta}\xi,
  \qquad \xi\sim N(0,I_d),
\]
a second-order Taylor expansion of \(f\in C_c^\infty(\R^d)\), together with
the Gaussian moment bounds and the uniform boundedness of \(b_\eta^A\) on
\(K\), gives uniformly for \(x\in K\),
\[
  Q_\eta^A f(x)-f(x)
  =
  \eta\langle b_\eta^A(x),\nabla f(x)\rangle
  +\beta^{-1}\eta\Delta f(x)
  +o_K(\eta).
\]
Dividing by \(\eta\) and using the uniform convergence \(b_\eta^A\to b\) on
\(K\) yields the claim.
\end{proof}

The condition \(\eta^\alpha A_\eta\to0\) locally uniformly ensures that the
denominator does not alter the limiting drift on compact sets.  Indeed, by the
local boundedness of \(b\),
\[
  \sup_{x\in K}
  \left\|
  \frac{b(x)}{1+\eta^\alpha A_\eta(x)}-b(x)
  \right\|
  \to0,\,\, \eta\downarrow0 .
\]
Without this local vanishing, the small-stepsize generator could converge to
one with a rescaled drift instead of \(\Lgen\).

\begin{theorem}[Weak stationary Gibbs consistency]\label{thm:weak-consistency}
Suppose the assumptions of Theorem~\ref{thm:invariant-measure} and Assumption~\ref{assump:gibbs-unique} hold.  Let \((\pi_\eta)_{0<\eta\le\eta_0}\) be an admissible family of invariant measures for \(\Qeta\) at the moment order required to control the generator-consistency terms below.  Then
\begin{equation}\label{eq:weak-conv}
  \pi_\eta\Rightarrow\piz,\qquad \eta\downarrow0 .
\end{equation}
Equivalently, for every sequence \(\eta_j\downarrow0\), the selected invariant measures satisfy \(\pi_{\eta_j}\Rightarrow\piz\).
\end{theorem}

\begin{proof}
Let \(\eta_j\downarrow0\).  By admissibility, the family
\((\pi_{\eta_j})\) is tight.  Hence, after passing to a subsequence, we may
assume
\[
  \pi_{\eta_j}\Rightarrow\pi_* .
\]
Let \(f\in C_c^\infty(\R^d)\).  By invariance of \(\pi_{\eta_j}\),
\[
  \int \frac{Q_{\eta_j}f-f}{\eta_j}\,\dd\pi_{\eta_j}=0.
\]

We claim that
\[
  \int \frac{Q_{\eta_j}f-f}{\eta_j}\,\dd\pi_{\eta_j}
  \longrightarrow
  \int \Lgen f\,\dd\pi_* .
\]
Let \(K\) be a compact neighborhood of \(\operatorname{supp} f\) large enough
so that \(\Lgen f=0\) on \(K^c\).  On \(K\), Lemma~\ref{lem:compact-generator-consistency}
gives
\[
  \frac{Q_{\eta_j}f-f}{\eta_j}\to \Lgen f
\]
uniformly.  Together with \(\pi_{\eta_j}\Rightarrow\pi_*\), this identifies the
contribution over \(K\).

It remains to control the contribution from \(K^c\).  Since \(f=0\) and
\(\Lgen f=0\) on \(K^c\), this contribution is bounded by
\[
  \int_{K^c}\frac{Q_{\eta_j}|f|(x)}{\eta_j}\,\pi_{\eta_j}(\dd x).
\]
Choosing \(K\) with positive distance from \(K^c\) to
\(\operatorname{supp} f\), the Gaussian proposal must make a jump of at least
that distance in order for \(f\) to be nonzero after one step.  The resulting
Gaussian tail, together with the admissible \(V_q\)-moment bound on
\(\pi_{\eta_j}\), gives
\[
  \int_{K^c}\frac{Q_{\eta_j}|f|(x)}{\eta_j}\,\pi_{\eta_j}(\dd x)
  \longrightarrow0 .
\]
Therefore
\[
  \int \Lgen f\,\dd\pi_*=0,
  \qquad f\in C_c^\infty(\R^d).
\]
By Assumption~\ref{assump:gibbs-unique}, \(\pi_*=\piz\).  Since every sequence
\(\eta_j\downarrow0\) has the same possible subsequential limit, the full
family satisfies
\[
  \pi_\eta\Rightarrow\piz .
\]
\end{proof}

Weak convergence only guarantees convergence of integrals against bounded continuous test functions \(f\in C_b(\R^d)\).  Some quantities used below, including empirical risks in the examples, may be unbounded.  We refer to these quantities as observables and record the standard compact--tail argument needed to compare their expectations under invariant measures.

\begin{lemma}[Weak convergence plus uniform integrability]\label{lem:ui-observable}
Let \(\mu_n\Rightarrow\mu\) on \(\R^d\).  Suppose the observable \(H:\R^d\to\R\) is continuous and has polynomial growth,
\[
  |H(x)|\le C_H(1+\|x\|^s),
\]
and suppose that for some \(\delta>0\),
\[
  \sup_n\int \|x\|^{s+\delta}\,\mu_n(\dd x)<\infty,
  \qquad
  \int \|x\|^{s+\delta}\,\mu(\dd x)<\infty.
\]
Then
\[
  \int H\,\dd\mu_n\to \int H\,\dd\mu.
\]
\end{lemma}

\begin{proof}
Fix \(R>0\), set \(K_R=\{x:\|x\|\le R\}\), and choose a continuous cutoff \(\chi_R\) with \(\chi_R=1\) on \(K_R\) and \(\chi_R=0\) outside \(K_{2R}\).  Then \(H_R:=H\chi_R\in C_b(\R^d)\), so
\[
  \int H_R\,\dd\mu_n\to \int H_R\,\dd\mu .
\]
The remaining terms are supported in \(K_R^c\).  From \(|H(x)|\le C_H(1+\|x\|^s)\) and the uniform \((s+\delta)\)-moment bound,
\[
  \sup_n\int_{K_R^c}|H(x)|\,\mu_n(\dd x)\to0,
  \qquad
  \int_{K_R^c}|H(x)|\,\mu(\dd x)\to0
  \qquad (R\to\infty).
\]
Letting \(n\to\infty\) for fixed \(R\), and then \(R\to\infty\), gives the result.
\end{proof}

\begin{corollary}[Polynomial stationary observable consistency]\label{cor:poly-observable-consistency}
Suppose the assumptions of Theorem~\ref{thm:weak-consistency} hold.  Let \(H\) be continuous and satisfy
\[
  |H(x)|\le C_H(1+\|x\|^s).
\]
If the admissible family of invariant measures satisfies a uniform \(s+\delta\) moment bound for some \(\delta>0\), and \(\piz\) has finite \(s+\delta\) moment, then
\[
  \int H\,\dd\pieta\to \int H\,\dd\piz,
  \qquad \eta\downarrow0.
\]
In particular, for empirical-risk observables with polynomial growth, the expectations under \(\pi_\eta\) converge to the expectation under \(\piz\) whenever the Lyapunov estimate provides a uniform moment bound of strictly higher order.
\end{corollary}

\begin{proof}
Apply Lemma~\ref{lem:ui-observable} with \(\mu_n=\pi_{\eta_n}\) and \(\mu=\piz\).
\end{proof}

This corollary gives the stationary interpretation used in this paper: fixed-stepsize taming may bias the invariant law, but the bias disappears for polynomial-growth observables as \(\eta\to0\), provided the admissible family has a uniform moment bound of order strictly higher than the growth order of the observable.

\section{Exact-Gradient Stationary Residual Decomposition}\label{sec:quantitative-residual}

Section~\ref{sec:weak-stationary-identification} proves the qualitative limit \(\pieta\Rightarrow\piz\) as \(\eta\downarrow0\).  This section studies what remains when the stepsize is fixed. For a probability measure \(\pi\) and an integrable observable \(H\), write
\[ \pi(H):=\int H\,\dd\pi . \] 

We compare the two stationary expectations
\[
  \pieta(H)-\piz(H).
\]
The Poisson equation rewrites this difference as the average one-step generator residual of the numerical kernel.  We then decompose that residual into an Euler discretization residual and a deterministic-envelope residual. 

\begin{proposition}[Poisson-residual identity for stationary bias]\label{prop:poisson-residual-identity}
Let \(H\) be an observable with \(\piz(H)\) finite.  Suppose the Poisson equation, written with the sign convention
\begin{equation}\label{eq:poisson-certificate}
  \Lgen u_H=H-\piz(H),
\end{equation}
has a solution \(u_H\) such that all integrals below are finite.  Let \(Q_\eta\) be any Markov kernel admitting an invariant probability measure \(\pieta\).  Then
\begin{equation}\label{eq:poisson-residual-identity}
  \pieta(H)-\piz(H)
  =
  \int\left(\Lgen u_H-\frac{Q_\eta u_H-u_H}{\eta}\right)\dd\pieta .
\end{equation}
Consequently,
\begin{equation}\label{eq:poisson-residual-bound}
  \left|\pieta(H)-\piz(H)\right|
  \le
  \int \left|\Lgen u_H-\frac{Q_\eta u_H-u_H}{\eta}\right|\dd\pieta .
\end{equation}
\end{proposition}

\begin{proof}
By invariance of \(\pieta\),
\[
  \int \frac{Q_\eta u_H-u_H}{\eta}\,\dd\pieta=0.
\]
Using \eqref{eq:poisson-certificate},
\[
  \pieta(H)-\piz(H)
  =\int \Lgen u_H\,\dd\pieta
  =\int\left(\Lgen u_H-\frac{Q_\eta u_H-u_H}{\eta}\right)\dd\pieta .
\]
The inequality follows by taking absolute values.
\end{proof}

This is the standard Poisson-equation device used in invariant-measure weak error analysis for numerical SDEs; see Mattingly--Stuart--Tretyakov~\cite{MattinglyStuartTretyakov2010}.  In the present setting, it turns stationary bias into a one-step generator residual, so that the effect of the deterministic envelope appears explicitly in the residual.

\begin{assumption}[Poisson regularity for residual bounds]\label{assump:poisson-regularity}
Fix an observable \(H\).  The Poisson equation
\[
  \Lgen u_H=H-\piz(H)
\]
admits a solution \(u_H\) with the following weighted \(C^{2,\gamma}\)-type bounds for some \(\gamma\in(0,1]\).  Namely, there exist constants \(C_H,m_H\) such that
\[
  \|\nabla u_H(x)\|+\|\nabla^2 u_H(x)\|
  \le C_H(1+\|x\|^{m_H}),
  \qquad x\in\mathbb R^d,
\]
and
\[
  \|\nabla^2u_H(x)-\nabla^2u_H(y)\|
  \le C_H(1+\|x\|^{m_H}+\|y\|^{m_H})\|x-y\|^\gamma,
  \qquad x,y\in\mathbb R^d .
\]
\end{assumption}

\begin{remark}[A route to Poisson regularity]\label{rem:poisson-quartic-route}
Assumption~\ref{assump:poisson-regularity} supplies the analytic regularity
needed for the Poisson-residual argument.  For ergodic elliptic diffusions,
results of Pardoux--Veretennikov~\cite{PardouxVeretennikov2001} provide a
standard route to Poisson solutions with polynomial-growth estimates.  In
smooth uniformly elliptic Langevin settings, these estimates can be combined
with local elliptic regularity, such as Schauder estimates
\cite{GilbargTrudinger2001}, to obtain weighted \(C^{2,\gamma}\)-type bounds of
the form assumed above.  We keep the assumption explicit because the precise
weights and H\"older exponent depend on the model.
\end{remark}

\begin{proposition}[Pointwise one-step residual bound]
\label{prop:onestep-residual-decomposition}
Assume Assumptions~\ref{assump:poly-lip} and
\ref{assump:poisson-regularity}.  Consider the exact-gradient
deterministic-envelope update
\[
  Y=x+\eta b_\eta(x)+\sqrt{2\beta^{-1}\eta}\,Z,
  \qquad
  b_\eta(x):=\frac{b(x)}{1+\eta^\alpha A_\eta(x)},
  \qquad
  Z\sim N(0,I_d),
\]
where \(A_\eta:\R^d\to[0,\infty)\) is deterministic.  Let \(Q_\eta\) be the
corresponding Markov kernel.  Then there exist constants \(C,m\) and
\(\eta_1>0\), depending on \(H\) but not on \(\eta\), such that for all
\(0<\eta\le\eta_1\),
\[
\left|
\frac{Q_\eta u_H(x)-u_H(x)}{\eta}
-\Lgen u_H(x)
\right|
\le C(1+\|x\|^m)
\left[
\eta^{\gamma/2}
+\eta^\alpha A_\eta(x)\|b(x)\|
\right].
\]
\end{proposition}

\begin{proof}
Write the one-step increment as
\[
  \Delta=\eta b_\eta(x)+\sqrt{2\beta^{-1}\eta}\,Z .
\]
Then
\[
  Q_\eta u_H(x)=\E[u_H(x+\Delta)] .
\]
A second-order Taylor expansion with the \(C^{2,\gamma}\)-type remainder from
Assumption~\ref{assump:poisson-regularity} gives
\[
  Q_\eta u_H(x)-u_H(x)
  =
  \langle\nabla u_H(x),\E\Delta\rangle
  +\frac12\E\big[\Delta^\top \nabla^2u_H(x)\Delta\big]
  +R_\eta(x),
\]
where, by Assumption~\ref{assump:poisson-regularity}, after increasing the polynomial weight if necessary,
\[
  |R_\eta(x)|=|\E \int_0^1(1-\tau)\, \Delta^\top \left[ \nabla^2u_H(x+\tau\Delta)-\nabla^2u_H(x) \right] \Delta\,\dd\tau|
  \le
  C\,\E\Big[(1+\|x\|^m+\|x+\Delta\|^m)\|\Delta\|^{2+\gamma}\Big].
\]
Since \(A_\eta\ge0\), we have \(\|b_\eta(x)\|\le\|b(x)\|\).  The polynomial
growth of \(b\), together with Gaussian moment bounds, therefore gives
\[
  \eta^{-1}|R_\eta(x)|
  \le C(1+\|x\|^m)\eta^{\gamma/2}.
\]

The first-order term satisfies
\[
  \eta^{-1}\langle\nabla u_H(x),\E\Delta\rangle
  =
  \langle b_\eta(x),\nabla u_H(x)\rangle .
\]
For the second-order term \(\frac12\eta^{-1}\E\big[\Delta^\top \nabla^2u_H(x)\Delta\big]\), expand
\[
  \Delta=\eta b_\eta(x)+\sqrt{2\beta^{-1}\eta}\,Z .
\]
The drift--noise cross term has zero expectation.  The pure Gaussian term gives
\[
  \frac{1}{2\eta}
  \E\left[
    (\sqrt{2\beta^{-1}\eta}Z)^\top
    \nabla^2u_H(x)
    (\sqrt{2\beta^{-1}\eta}Z)
  \right]
  =
  \beta^{-1}\operatorname{tr}(\nabla^2u_H(x))
  =
  \beta^{-1}\Delta u_H(x),
\]
where we used \(\E[ZZ^\top]=I_d\).

  The deterministic drift-square contribution is
\[
  \frac{\eta}{2}b^\top_\eta(x)\nabla^2u_H(x)b_\eta(x).
\]
Using the polynomial bounds on \(\nabla^2u_H\), the polynomial growth of \(b\),
and \(\|b_\eta(x)\|\le\|b(x)\|\), this term is bounded by
\[
  C(1+\|x\|^m)\eta .
\]
Since \(0<\gamma\le1\), we have \(\eta\le\eta^{\gamma/2}\) for
\(0<\eta\le1\).  Hence this term is also bounded by
\[
  C(1+\|x\|^m)\eta^{\gamma/2}.
\]

Combining the preceding estimates yields
\[
  \frac{Q_\eta u_H(x)-u_H(x)}{\eta}
  =
  \langle b_\eta(x),\nabla u_H(x)\rangle
  +\beta^{-1}\Delta u_H(x)
  +\mathcal R_\eta(x),
\]
with
\[
  |\mathcal R_\eta(x)|
  \le C(1+\|x\|^m)\eta^{\gamma/2}.
\]
Since
\[
  \Lgen u_H(x)
  =
  \langle b(x),\nabla u_H(x)\rangle
  +\beta^{-1}\Delta u_H(x),
\]
the Laplacian terms cancel in \(\frac{Q_\eta u_H(x)-u_H(x)}{\eta}
-\Lgen u_H(x)\).  The remaining drift difference is
\[
  b_\eta(x)-b(x)
  =
  -\frac{\eta^\alpha A_\eta(x)}{1+\eta^\alpha A_\eta(x)}\,b(x).
\]
Therefore
\[
  |\langle b_\eta(x)-b(x),\nabla u_H(x)\rangle|
  \le
  C(1+\|x\|^m)\eta^\alpha A_\eta(x)\|b(x)\|.
\]
Combining this bound with the estimate on \(\mathcal R_\eta\) proves the
claim.
\end{proof}

\begin{corollary}[Fixed-stepsize stationary bias decomposition]\label{cor:stationary-bias-decomposition}
Under the assumptions of Proposition~\ref{prop:onestep-residual-decomposition}, let \(\pi_\eta\) be an invariant probability measure of \(Q_\eta\) that integrates the polynomial weights and residual terms appearing on the right-hand side below.  Then
\begin{equation}\label{eq:stationary-bias-decomposition}
\left|\pi_\eta(H)-\piz(H)\right|
\le C_H\int (1+\|x\|^m)
\left[
\eta^{\gamma/2}
+\eta^\alpha A_\eta(x)\|b(x)\|
\right]\pi_\eta(\dd x).
\end{equation}
\end{corollary}

\begin{proof}
Combine Proposition~\ref{prop:poisson-residual-identity} with Proposition~\ref{prop:onestep-residual-decomposition}.
\end{proof}

The bound separates two residuals.  The term \(\eta^{\gamma/2}\) is the Euler residual from the Taylor remainder in the one-step expansion of \(Q_\eta u_H\).  The term \(\eta^\alpha A_\eta(x)\|b(x)\|\) is the envelope residual.  It is paid only where the denominator is active.  Thus localized denominators reduce the envelope residual in the typical region while retaining a tail correction for stability.  This observation is structural; it does not by itself rank invariant measures produced by different kernels.

\section{Deterministic-Envelope Stochastic-Gradient Extension}\label{sec:sg-extension}

Section~\ref{sec:quantitative-residual} decomposed the fixed-stepsize residual for the exact-gradient deterministic-envelope kernel into an Euler residual and a deterministic-envelope residual. This section extends the deterministic-envelope kernel to stochastic-gradient oracles. The key point is that the denominator is fixed conditional on the current state, so a conditionally unbiased oracle remains conditionally unbiased after deterministic scaling.

Let
\begin{equation}\label{eq:sg-oracle}
  g(x,U)=\nabla F_z(x)+\zeta(x,U),
  \qquad \E[\zeta(x,U)\mid x]=0.
\end{equation}
Here \(U\) denotes the oracle randomness, such as a random index or mini-batch.  Its law may depend on the data \(z\).

The deterministic-envelope stochastic-gradient update is
\begin{equation}\label{eq:sg-tamed-kernel}
  Y_{k+1}^{\mathrm{sg}}
  =Y_k^{\mathrm{sg}}
  -\eta\frac{g(Y_k^{\mathrm{sg}},U_{k+1})}{1+\eta^\alpha A_\eta(Y_k^{\mathrm{sg}})}
  +\sqrt{2\beta^{-1}\eta}\,\xi_{k+1}.
\end{equation}
Since \(b=-\nabla F_z\), the tamed stochastic drift can be written as
\begin{equation}\label{eq:sg-tamed-drift-decomposition}
  \widehat b^T_\eta(x,U)
  :=\frac{b(x)-\zeta(x,U)}{1+\eta^\alpha A_\eta(x)}
  =b^T_\eta(x)+\varepsilon_\eta(x,U),
  \qquad
  \varepsilon_\eta(x,U):=-\frac{\zeta(x,U)}{1+\eta^\alpha A_\eta(x)}.
\end{equation}
Because the denominator is deterministic conditional on \(x\),
\begin{equation}\label{eq:sg-centered-tamed-drift}
  \E[\widehat b^T_\eta(x,U)\mid x]=b^T_\eta(x),
  \qquad
  \E[\varepsilon_\eta(x,U)\mid x]=0.
\end{equation}
This identity will be used in both the Lyapunov estimate and the stationary residual decomposition below.

A mini-batch-dependent denominator generally does not have this property.
For example, let \(\widetilde A_\eta(x,U)\) be a nonnegative quantity computed
from the current stochastic-gradient realization \(g(x,U)\).  Then typically
\begin{equation}\label{eq:random-denominator-centering-failure}
  \E\left[
  \frac{g(x,U)}{1+\eta^\alpha \widetilde A_\eta(x,U)}
  \mid x
  \right]
  \ne
  b^T_\eta(x).
\end{equation}
 
The stochastic-gradient extension keeps the denominator deterministic, but the update now contains oracle noise. 
Two remaining assumptions are the stochastic-gradient analogues of the polynomial-growth and effective-linearity inputs used for the deterministic drift.

\begin{assumption}[Oracle moments and envelope-compatible increments]\label{assump:sg-oracle}
The stochastic-gradient oracle satisfies the following conditions.
\begin{enumerate}[label=\textup{(\roman*)},leftmargin=2.1em]

\item \textup{(Polynomial conditional moments.)}  Fix a maximal moment order \(r\ge2\) large enough to include all Lyapunov and one-step generator expansion orders used below.  For each \(2\le j\le r\), the centered oracle error satisfies
\begin{equation}\label{eq:sg-raw-moments}
  \E[\|\zeta(x,U)\|^j\mid x]
  \le C_j\delta^{j/2}(1+\|x\|^{s_j}),
\end{equation}
where \(0\le\delta\le\delta_0<\infty\).

\item \textup{(Envelope-compatible stochastic increments.)}  For each moment order \(q\le r\) used in the Lyapunov estimate, the envelope-scaled oracle increment satisfies, for \(2\le j\le q\),
\begin{equation}\label{eq:sg-compatible-increment}
  \E\left\|
  \eta\frac{\zeta(x,U)}{1+\eta^\alpha A_\eta(x)}
  \right\|^j
  \le C_{q,j}\eta^{j(1-\alpha)}\delta^{j/2}(1+\|x\|^j).
\end{equation}
\end{enumerate}
\end{assumption}

With this oracle condition in place, the Lyapunov stability estimate extends to the stochastic-gradient kernel.  The only new term is the envelope-scaled oracle increment, whose conditional mean and moments are controlled by Assumption~\ref{assump:sg-oracle}.

\begin{proposition}[Moment stability for the stochastic-gradient chain]\label{prop:sg-moment}
Assume Assumptions~\ref{assump:poly-lip}--\ref{assump:envelope} and Assumption~\ref{assump:sg-oracle}; alternatively, use any localized deterministic envelope satisfying the effective-linearity and polynomial-growth bounds used in Proposition~\ref{prop:moment-stability}.  For every \(q\ge2\) for which \eqref{eq:sg-compatible-increment} holds, the stochastic-gradient tamed chain \eqref{eq:sg-tamed-kernel} satisfies the same type of Lyapunov estimate as \eqref{eq:drift-condition}: for all sufficiently small \(\eta\),
\begin{equation}\label{eq:sg-drift-condition}
  \Qsg V_q(x)
  \le (1-c_q\eta)V_q(x)+C_q\eta,
  \qquad V_q(x)=1+\|x\|^q.
\end{equation}
Consequently \(\Qsg\) admits invariant measures.  Moreover, any selected family of invariant measures satisfying \eqref{eq:sg-drift-condition} is admissible at the corresponding moment order for the consistency arguments below.
\end{proposition}

\begin{proof}
We only indicate the modifications to the proof of Proposition~\ref{prop:moment-stability}.  Write the stochastic-gradient increment as
\[
  \Delta_\eta^{\mathrm{sg}}
  =
  \eta b^T_\eta(x)
  +
  \eta\varepsilon_\eta(x,U)
  +
  \sqrt{2\beta^{-1}\eta}\,\xi .
\]
The deterministic drift and Gaussian parts are the same as in Proposition~\ref{prop:moment-stability}.  The only new term is the envelope-scaled oracle increment \(\eta\varepsilon_\eta(x,U)\).

The first-order Lyapunov contribution of this oracle increment vanishes.  By \eqref{eq:sg-centered-tamed-drift},
\[
  \E[\varepsilon_\eta(x,U)\mid x]=0.
\]
Hence the first-order drift term remains
\[
  q\eta\|x\|^{q-2}\langle x,b^T_\eta(x)\rangle,
\]
the same deterministic-envelope dissipative term estimated in Proposition~\ref{prop:moment-stability}.

It remains to control the oracle contribution to the polynomial remainders.  Assumption~\ref{assump:sg-oracle}\textup{(ii)} gives, for the moment orders used in the Lyapunov expansion,
\[
  \E\|\eta\varepsilon_\eta(x,U)\|^j
  \le
  C_{q,j}\eta^{j(1-\alpha)}\delta^{j/2}(1+\|x\|^j),
  \qquad 2\le j\le q.
\]
Since \(\delta\le\delta_0\), the factor \(\delta^{j/2}\) is absorbed into the constant below.
Thus the oracle increment has the same effective size as the deterministic tamed drift increment in the remainder estimates.  Combining this bound with the deterministic and Gaussian increment estimates from Proposition~\ref{prop:moment-stability}, the same absorption argument yields \eqref{eq:sg-drift-condition}.  The invariant-measure existence and uniform moment bound then follow exactly as in Theorem~\ref{thm:invariant-measure}.
\end{proof}

Before stating the weak consistency result, define the conditional second moment of the envelope-scaled oracle perturbation by
\begin{equation}\label{eq:sg-oracle-second-moment}
\begin{split}
  \mathsf{Var}_x(\varepsilon_\eta)
  &:=
  \E[\|\varepsilon_\eta(x,U)\|^2\mid x]  \\
  &=
  \E\left[
  \left\|
  \frac{\zeta(x,U)}{1+\eta^\alpha A_\eta(x)}
  \right\|^2
  \mid x\right].
\end{split}
\end{equation}
This scalar second moment is the noise scale used in the stochastic-gradient residual bounds below.  For random-denominator schemes, the same conditional second-moment diagnostic measures the oracle-noise part of the stationary residual.

Proposition~\ref{prop:sg-moment} gives the stability and tightness needed to identify stationary limits.  We now identify the small-step limits of the stochastic-gradient invariant measures.

\begin{theorem}[Weak stationary consistency for deterministic-envelope SGLD]\label{thm:sg-weak-consistency}
Assume the hypotheses of Proposition~\ref{prop:sg-moment}, Lemma~\ref{lem:compact-generator-consistency} for the deterministic part, and Assumption~\ref{assump:gibbs-unique}.  Let \(\eta_j\downarrow0\), and let \(\pi_{\eta_j}^{\mathrm{sg}}\) be any invariant measure from an admissible family for the stochastic-gradient kernel \(Q_{\eta_j}^{\mathrm{sg}}\).  Suppose that, for a polynomial weight \(W\) large enough to control the weighted remainders in the one-step generator expansion used in Proposition~\ref{prop:onestep-residual-decomposition},
\begin{equation}\label{eq:sg-weak-covariance-condition}
  \eta_j\int W(x)\,\mathsf{Var}_x(\varepsilon_{\eta_j})\,
  \pi_{\eta_j}^{\mathrm{sg}}(\dd x)
  \longrightarrow0,
\end{equation}
and that the moment bounds from Proposition~\ref{prop:sg-moment} hold at a sufficiently high order to integrate these remainders. Then
\[
  \pi_{\eta_j}^{\mathrm{sg}}\Rightarrow \piz .
\]
\end{theorem}

\begin{proof}
 The Lyapunov estimate in Proposition~\ref{prop:sg-moment} gives tightness of \((\pi_\eta^{\mathrm{sg}})_\eta\). Let \(\eta_j\downarrow0\), and pass to a weakly convergent subsequence, still denoted by \(\eta_j\), such that
 
 \[ \pi_{\eta_j}^{\mathrm{sg}}\Rightarrow \mu . \]
 By invariance of \(\pi_{\eta_j}^{\mathrm{sg}}\) for \(Q_{\eta_j}^{\mathrm{sg}}\), \[ \int \frac{Q_{\eta_j}^{\mathrm{sg}}f-f}{\eta_j} \,\dd\pi_{\eta_j}^{\mathrm{sg}}=0, \qquad f\in C_c^\infty(\R^d). \]  We next show that
\[
  \int
  \left[
    \frac{Q_{\eta_j}^{\mathrm{sg}}f-f}{\eta_j}
    -
    \Lgen f
  \right]
  \,\dd\pi_{\eta_j}^{\mathrm{sg}}
  \longrightarrow 0 .
\]
For compactly supported smooth \(f\), Lemma~\ref{lem:compact-generator-consistency} gives the deterministic-envelope part of the local generator convergence. The stochastic-gradient increment adds the centered term \(-\eta_j\varepsilon_{\eta_j}(x,U)\). Its first-order Taylor contribution vanishes by the centering identity \eqref{eq:sg-centered-tamed-drift}. Its second-order contribution is bounded by
 \[ C_f\,\eta_j\,\mathsf{Var}_x(\varepsilon_{\eta_j}), \] 
 because \(f\in C_c^\infty(\R^d)\) has bounded Hessian. The terms not involving the oracle covariance vanish as \(j\to\infty\). More precisely, the deterministic-envelope generator-consistency error, together with the higher-order oracle Taylor remainders controlled by Assumption~\ref{assump:sg-oracle} and the uniform moment bounds, contributes a quantity \(r_j(f)\) with \(r_j(f)\to0\). Hence, for a polynomial weight \(W\) depending on \(f\) and on the moment bounds, 
 \[ \left| \int \left[ \frac{Q_{\eta_j}^{\mathrm{sg}}f-f}{\eta_j} - \Lgen f \right] \,\dd\pi_{\eta_j}^{\mathrm{sg}} \right| \le r_j(f) + C\eta_j \int W(x)\mathsf{Var}_x(\varepsilon_{\eta_j}) \,\pi_{\eta_j}^{\mathrm{sg}}(\dd x). \] 
 The last term vanishes by the covariance condition \eqref{eq:sg-weak-covariance-condition}. Therefore 
 \[ \int \Lgen f\,\dd\mu=0, \qquad f\in C_c^\infty(\R^d). \] 
 Assumption~\ref{assump:gibbs-unique} identifies \(\mu=\piz\). Since every weakly convergent subsequence has the same limit, the full family converges weakly to \(\piz\). 
 \end{proof}

The covariance condition in Theorem~\ref{thm:sg-weak-consistency} has a simple
mini-batch interpretation.  Suppose the mini-batch oracle noise satisfies
\[
  \mathsf{Var}_x(\zeta_{\ell_b})
  \le
  \frac{C}{\ell_b}(1+\|x\|^s).
\]
Since the deterministic envelope only divides the oracle error by
\(1+\eta^\alpha A_\eta(x)\ge1\), it cannot increase the conditional second
moment:
\[
  \mathsf{Var}_x(\varepsilon_\eta)
  \le
  \mathsf{Var}_x(\zeta_{\ell_b}).
\]
If the invariant measures have a uniform moment bound for the weight
\(W(x)(1+\|x\|^s)\), then
\[
\begin{aligned}
\eta_j
\int W(x)\mathsf{Var}_x(\varepsilon_{\eta_j})
\,\pi_{\eta_j}^{\rm sg}(\dd x)
\le
\frac{C\eta_j}{\ell_b}
\int W(x)(1+\|x\|^s)\,
\pi_{\eta_j}^{\rm sg}(\dd x) 
\le
C'\frac{\eta_j}{\ell_b}.
\end{aligned}
\]
Thus the weighted covariance condition holds whenever
\(\eta_j/\ell_b\to0\).  In particular, it holds for any fixed mini-batch size
\(\ell_b\ge1\) as \(\eta_j\to0\).

\begin{corollary}[Stochastic-gradient polynomial-observable consistency]\label{cor:sg-observable-consistency}
Under the assumptions of Theorem~\ref{thm:sg-weak-consistency}, let \(H\) be a continuous polynomial-growth observable.  If the invariant measures have a uniform moment bound of order strictly larger than the growth order of \(H\), then
\[
  \int H\,\dd\pi_{\eta_j}^{\mathrm{sg}}
  \to
  \int H\,\dd\piz .
\]
\end{corollary}

\begin{proof}
Apply Lemma~\ref{lem:ui-observable} to the weak convergence in Theorem~\ref{thm:sg-weak-consistency} and to the stated uniform moment bound.
\end{proof}

The consistency result above identifies the small-stepsize limit.  We next bound the fixed-stepsize stationary bias by one-step residuals for Poisson observables.

\begin{theorem}[Stationary bias bound through stochastic-gradient residuals]\label{thm:sg-stationary-error-decomposition}
Assume the Poisson regularity condition in Assumption~\ref{assump:poisson-regularity}.  Assume also that Assumption~\ref{assump:sg-oracle} supplies the finite conditional moment order needed for the polynomially weighted Taylor remainder below, and that \(b\) has the polynomial growth implied by Assumption~\ref{assump:poly-lip}.  Let \(\pi_\eta^{\mathrm{sg}}\) be an invariant probability measure of \(\Qsg\) for which the right-hand side below is integrable.  Then there exist constants \(C_H,m\), independent of sufficiently small \(\eta\), such that
\begin{equation}\label{eq:sg-stationary-bias-decomposition}
\left|\pi_\eta^{\mathrm{sg}}(H)-\piz(H)\right|
\le C_H\int (1+\|x\|^m)
\left[
\eta^{\gamma/2}
+\eta^\alpha A_\eta(x)\|b(x)\|
+\eta\,\mathsf{Var}_x(\varepsilon_\eta)
\right]\pi_\eta^{\mathrm{sg}}(\dd x).
\end{equation}
\end{theorem}

\begin{proof}
We first record the local residual estimate.  The argument is the proof of Proposition~\ref{prop:onestep-residual-decomposition} with only the new oracle terms retained.  Write the stochastic-gradient increment as
\[
  \Delta
  =
  \eta b^T_\eta(x)+\eta\varepsilon_\eta(x,U)
  +\sqrt{2\beta^{-1}\eta}\,Z,
  \qquad Z\sim N(0,I_d).
\]
The deterministic drift part and the Gaussian covariance are the same as in the exact-gradient residual decomposition.  The new first-order oracle term vanishes because
\[
  \E[\varepsilon_\eta(x,U)\mid x]=0.
\]
Thus the deterministic-envelope drift residual remains
\[
  |\langle b^T_\eta(x)-b(x),\nabla u_H(x)\rangle|
  \le C_H(1+\|x\|^m)\eta^\alpha A_\eta(x)\|b(x)\|.
\]

It remains to bound the new second-order and higher-order oracle terms.  The mixed deterministic--oracle term vanishes conditionally by centering, and the mixed Gaussian--oracle term vanishes by independence and centering.  The pure oracle covariance is estimated directly by the Hessian growth bound:
\[
\begin{aligned}
\left|
  \frac{\eta}{2}
  \E\!\left[
    \varepsilon_\eta(x,U)^\top
    \nabla^2u_H(x)
    \varepsilon_\eta(x,U)
    \mid x
  \right]
\right|
&\le
  \frac{\eta}{2}\|\nabla^2u_H(x)\|_{\mathrm{op}}
  \E\!\left[\|\varepsilon_\eta(x,U)\|^2\mid x\right] \\
&\le
  C_H(1+\|x\|^m)\eta\,\mathsf{Var}_x(\varepsilon_\eta).
\end{aligned}
\]

The second line uses the polynomial Hessian bound from Assumption~\ref{assump:poisson-regularity}.  The remaining Taylor remainder is controlled by the same weighted \(C^{2,\gamma}\) argument as in Proposition~\ref{prop:onestep-residual-decomposition}, using the oracle moment bounds in Assumption~\ref{assump:sg-oracle}; after division by \(\eta\), it is bounded by
\[
  C_H(1+\|x\|^m)\eta^{\gamma/2}.
\]
Combining these estimates gives the pointwise local residual bound
\[
\left|
\frac{\Qsg u_H(x)-u_H(x)}{\eta}
-\Lgen u_H(x)
\right|
\le C_H(1+\|x\|^m)
\left[
\eta^{\gamma/2}
+\eta^\alpha A_\eta(x)\|b(x)\|
+\eta\,\mathsf{Var}_x(\varepsilon_\eta)
\right].
\]
Applying the Poisson-residual identity in Proposition~\ref{prop:poisson-residual-identity} with \(Q_\eta=\Qsg\), and then integrating this local bound against \(\pi_\eta^{\mathrm{sg}}\), gives \eqref{eq:sg-stationary-bias-decomposition}.
\end{proof}

Theorem~\ref{thm:sg-stationary-error-decomposition} bounds stationary bias by
an integrated one-step residual.  The bound has three residual terms.  The term
\(\eta^{\gamma/2}\) is the Euler residual from the local Taylor remainder.  The
term \(\eta^\alpha A_\eta(x)\|b(x)\|\) is the envelope residual from replacing
\(b\) by \(b_\eta^T\).  The term
\[
  \eta\,\mathsf{Var}_x(\varepsilon_\eta)
\]
is the oracle-covariance residual.  Because the envelope is deterministic
conditional on \(x\), there is no extra first-order mean-shift residual.  Thus
the stationary bias is controlled by Euler, envelope, and oracle-covariance
residuals, rather than by an additional oracle-coupled drift shift.

\section{Numerical Experiments}\label{sec:experiments}

The experiments address three questions suggested by the theory.  First, does using the current stochastic-gradient sample in the denominator create an additional mean-drift bias?  Second, how do different deterministic envelopes affect stability and stationary distortion?  Third, do the same mechanisms persist beyond radial toy targets, including regression examples, implementable envelope rules, and higher-dimensional stress tests?

\subsection{Experimental setup}\label{subsec:numerical-protocol}

The experiments isolate how denominator design affects stationary error.  We compare gradient-dependent denominators, global deterministic taming, and localized deterministic-envelope designs on controlled high-growth and empirical-risk examples.  Reported errors are absolute gaps to a reference stationary expectation.  For
radial targets this reference is available analytically; for regression targets we use a long exact-gradient chain as the numerical reference.

When a numerical reference is used, we account for its Monte Carlo uncertainty.  For an error of the form
\[
  \Delta H=\widehat H_{\rm method}-\widehat H_{\rm ref},
\]
we use the propagated standard error
\[
  \mathrm{SE}(\Delta H)
  =
  \left(
  \mathrm{SE}_{\rm method}(H)^2+
  \mathrm{SE}_{\rm ref}(H)^2
  \right)^{1/2}.
\]
Errors comparable to this scale are interpreted as reaching the numerical-reference uncertainty level, rather than as genuine stationary bias.

All denominator parameters are fixed before the reported sampling runs.  In the radial stress tests, the denominator scales and localized-envelope thresholds are fixed by the stated stress-test protocol, and the analytic reference values are used only to report errors.  In the quartic-regression examples, the localized thresholds are chosen from warm-up quantiles of a deterministic
growth indicator before the evaluation runs.  In the (d=50) hybrid stress test, the warm-up quantiles are combined with the stated small grid for \(R,S,\theta\).  The polynomial envelope is set from the objective growth bound.  Thus the reported tables are evaluation runs rather than post hoc parameter searches.

\subsection{Bias from a random denominator}\label{subsec:denominator-coupling}

\medskip
\noindent\textbf{Experiment 1: One-step random-denominator bias.}\par
\smallskip
This short diagnostic illustrates the local mechanism from
Propositions~\ref{prop:gradient-dependent-denominator-bias} and~\ref{prop:vector-gradient-dependent-denominator-bias}.  We use the scalar unbiased oracle
\[
  g_b=\mu+\sigma b^{-1/2}Z,
  \qquad Z\sim N(0,1),
\]
with \(\mu=1\) and \(\sigma=3\), and compare a random-denominator transform with its deterministic-denominator counterpart,
\[
  T_\lambda^{\rm rand}(g)=\frac{g}{1+\lambda |g|},
  \qquad
  T_\lambda^{\rm det}(\mu)=\frac{\mu}{1+\lambda |\mu|}.
\]
The diagnostic compares the conditional mean distortion of the random
denominator with the Monte Carlo baseline for the deterministic denominator.
The random-denominator distortion is
\[
\left|
\E T_\lambda^{\rm rand}(g_b)-T_\lambda^{\rm det}(\mu)
\right|.
\]
For the deterministic denominator, the corresponding distortion is zero in expectation because the denominator is fixed conditional on the current state; the small nonzero curve in the figure is only the finite-sample Monte Carlo error.

\begin{figure}[H]
\centering
\includegraphics[width=0.58\textwidth]{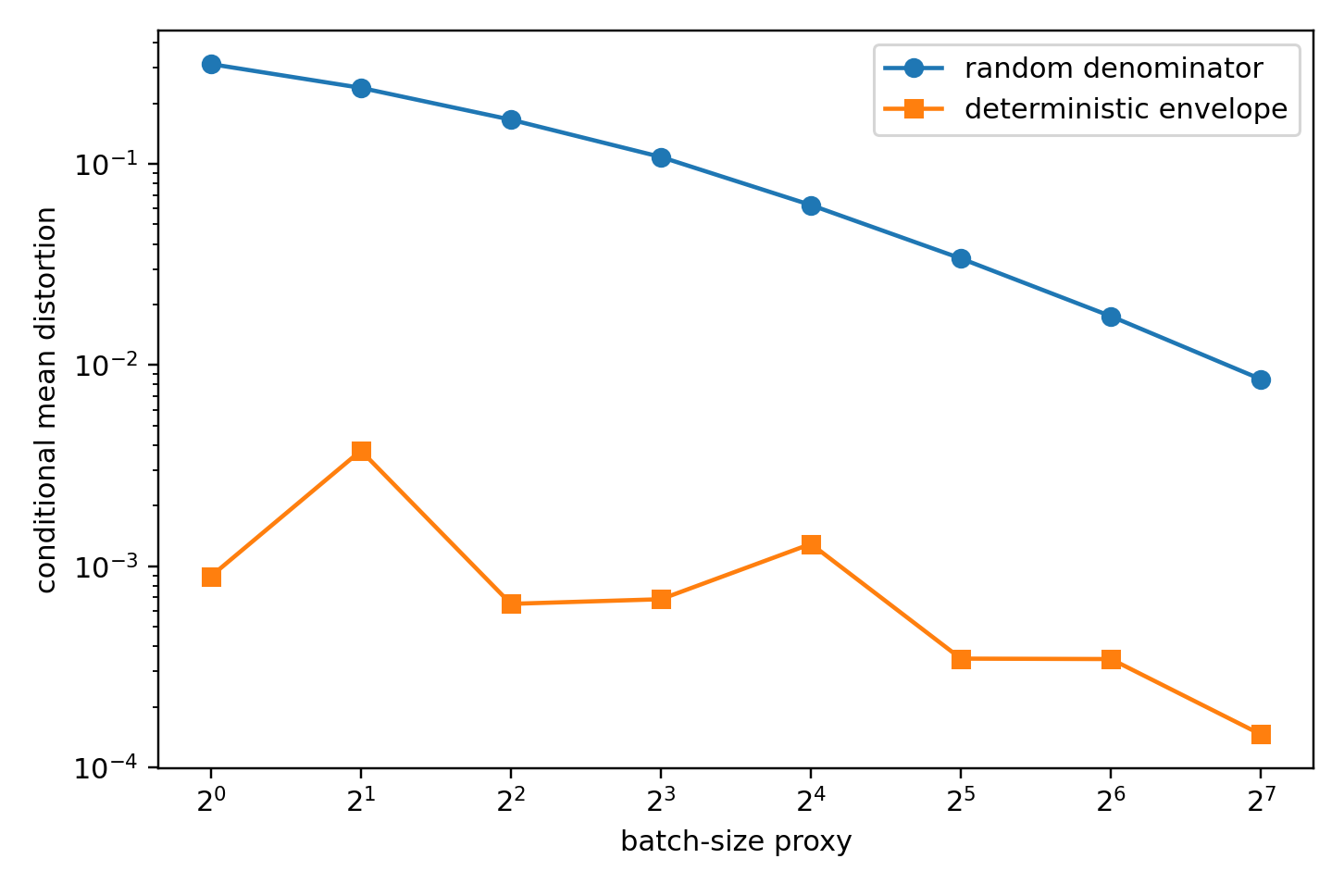}
\caption{One-step mean drift shift caused by a random denominator.}
\label{fig:oracle-bias-diagnostic}
\end{figure}

Experiment~1 is a local diagnostic: it shows that a random denominator can create a conditional mean shift even when the stochastic-gradient oracle is unbiased.  This does not yet show that the effect matters for the invariant law.  Experiment~2 therefore asks whether the same mean-shift mechanism is visible at stationarity.

\medskip
\noindent\textbf{Experiment 2: Stationary observable errors for the quartic target.}\par
\smallskip
We now test the same denominator effect at the stationary-observable level.  Before specifying the target, we first define the diagnostics used in the table.  Let \(\widehat b_\eta(x,U)\) denote the tamed stochastic drift used by the method being tested, and let \(b_\eta^T(x)\) be the deterministic-envelope drift.  The reported mean drift shift is
\[
  \E_x\left\|\E_U[\widehat b_\eta(x,U)\mid x]-b_\eta^T(x)\right\|.
\]
It measures the extra first-order oracle-bias channel that appears for a random denominator and is removed by a deterministic denominator.

The reported \(\eta\)-noise variance is an empirical estimate of the
second-order covariance term
\(\eta\int \mathsf{Var}_x(\varepsilon_\eta)\pi_\eta(\dd x):
\)

\[
  \eta\,\E_x\E_U\left[
  \left\|\widehat b_\eta(x,U)-\E_U[\widehat b_\eta(x,U)\mid x]\right\|^2
  \right].
\]
It measures the scale of the covariance-level stochastic-gradient contribution. Finally, the observable columns report the stationary errors
\[
|\Delta E H|
=
\left|
E_{\pi^{\rm method}_\eta}H-E_{\rm ref}H
\right|,
\]
for \(H=x^2, F\). 
The moment \(H=x^2\) tracks the scale of the sampled invariant law, while \(H=F=x^4/4\) tracks the average potential under that law. Together they check whether the drift-level denominator effect is visible in the invariant measure, rather than only in the one-step diagnostics.

We instantiate these diagnostics on the one-dimensional quartic target
\[
  F(x)=\frac{x^4}{4},\qquad b(x)=-x^3,
\]
for which the Gibbs moments are available analytically.  The stochastic-gradient oracle is
\[
  g(x,U)=b(x)+\sigma U,
  \qquad U\sim N(0,1),
\]
with \(\sigma=2.5\), and the stepsize is \(\eta=0.04\).  We compare exact deterministic-envelope taming, stochastic-gradient taming with the same deterministic denominator, and stochastic-gradient taming with a denominator recomputed from the current oracle realization.  Each row averages ten independent runs; the errors are absolute deviations from the analytic Gibbs values.

\begin{table}[H]
\centering
\scriptsize
\caption{Stationary errors for random and deterministic denominators.}
\label{tab:observable-denominator-coupling}
\resizebox{\textwidth}{!}{%
\begin{tabular}{lcccc}
\toprule
Method & mean drift shift & \(\eta\)-noise var. & \(|\Delta \E x^2|\) & \(|\Delta \E F|\) \\
\midrule
Exact-gradient deterministic envelope
& \(0\)
& \(0\)
& \(0.233\pm0.001\)
& \(0.273\pm0.003\) \\
SG, deterministic denominator
& \(0\)
& \(0.187\pm0.0001\)
& \(0.297\pm0.002\)
& \(0.334\pm0.004\) \\
SG, random denominator
& \(0.140\pm0.0002\)
& \(0.0727\pm0.00004\)
& \(0.362\pm0.003\)
& \(0.399\pm0.004\) \\
\bottomrule
\end{tabular}%
}
\end{table}

Table~\ref{tab:observable-denominator-coupling} shows the denominator effect at the level of stationary averages. With a deterministic denominator, the measured mean drift shift is zero, although the \(\eta\)-noise variance is larger than for the random-denominator chain. The random denominator therefore reduces the covariance-level noise diagnostic, but it also creates a nonzero mean drift shift. For the two reported observables, \(x^2\) and \(F\), the resulting stationary errors are smaller for the deterministic denominator. Thus, in this experiment, the random denominator's smaller noise variance is outweighed by the additional mean-shift channel.

\medskip
\noindent\textbf{Experiment 3: Batch-size scaling of the stochastic-gradient error.}\par
\smallskip
The previous experiment uses a fixed oracle noise level.  Experiment~3 checks how the drift-level diagnostics change when the oracle noise is reduced.  As discussed in Section~7, if the mini-batch oracle noise has conditional variance of order \(1/\ell_b\), then the covariance-level residual scale \(\eta\operatorname{Var}_x(\varepsilon_\eta)\) should decrease at the same order.  We vary the batch-size proxy while keeping the same quartic target and stepsize \(\eta=0.04\).  The stochastic-gradient oracle is
\[
  g_{\ell_b}(x,U)=b(x)+\sigma \ell_b^{-1/2}U,
  \qquad U\sim N(0,1),
\]
so increasing \(\ell_b\) reduces only the oracle noise.  Thus the \(\eta\)-scaled noise variance for the deterministic-denominator scheme should decrease approximately like \(1/\ell_b\), while the mean drift shift should remain absent.  For the random-denominator scheme, the denominator-induced mean drift shift should also decrease as the oracle noise is reduced, but it need not vanish at finite \(\ell_b\).

\begin{table}[H]
\centering
\scriptsize
\caption{Batch-size scaling for the quartic target.}
\label{tab:batch-size-scaling}
\begin{tabular}{ccccccc}
\toprule
\(\ell_b\) & \multicolumn{2}{c}{Deterministic denominator} & \multicolumn{2}{c}{Random denominator} & \multicolumn{2}{c}{\(\lvert\Delta \E F\rvert\)} \\
\cmidrule(lr){2-3}\cmidrule(lr){4-5}\cmidrule(lr){6-7}
& mean drift shift & \(\eta\)-noise var. & mean drift shift & \(\eta\)-noise var. & det. denom. & rand. denom. \\
\midrule
1 & \(0\) & \(0.187 \pm 0.0001\) & \(0.140 \pm 0.0003\) & \(0.0726 \pm 5\times10^{-5}\) & \(0.339 \pm 0.003\) & \(0.405 \pm 0.004\) \\
2 & \(0\) & \(0.0944 \pm 7\times10^{-5}\) & \(0.0919 \pm 0.0002\) & \(0.0450 \pm 4\times10^{-5}\) & \(0.312 \pm 0.004\) & \(0.354 \pm 0.004\) \\
4 & \(0\) & \(0.0474 \pm 4\times10^{-5}\) & \(0.0580 \pm 0.0001\) & \(0.0266 \pm 2\times10^{-5}\) & \(0.299 \pm 0.003\) & \(0.323 \pm 0.003\) \\
8 & \(0\) & \(0.0238 \pm 2\times10^{-5}\) & \(0.0355 \pm 5\times10^{-5}\) & \(0.0151 \pm 1\times10^{-5}\) & \(0.282 \pm 0.003\) & \(0.298 \pm 0.004\) \\
16 & \(0\) & \(0.0119 \pm 1\times10^{-5}\) & \(0.0210 \pm 4\times10^{-5}\) & \(0.00826 \pm 8\times10^{-6}\) & \(0.278 \pm 0.003\) & \(0.297 \pm 0.003\) \\
\bottomrule
\end{tabular}
\end{table}

Table~\ref{tab:batch-size-scaling} shows two effects.  For the deterministic denominator, the mean drift shift is structurally zero and the \(\eta\)-scaled noise variance is nearly halved whenever \(\ell_b\) is doubled.  For the random denominator, increasing \(\ell_b\) also reduces the extra mean drift shift, but this first-order bias is not absent at finite batch size.  The observable error \(\lvert\Delta \E F\rvert\) moves in the same direction, so the batch-size trend is visible both at the drift level and at the stationary-observable level.

\subsection{Choosing the deterministic envelope}\label{subsec:tamed-family-distortion}

\medskip
\noindent\textbf{Experiment 4: Why hybrid localization is needed.}\par
\smallskip

This experiment stress-tests the deterministic-envelope design in a radial
model where the reference risk is known.  We use
\begin{equation}\label{eq:radial-stress-potential}
  F_p(x)=\frac{\|x\|^p}{p},
  \qquad x\in\mathbb R^d,
\end{equation}
with \(d=5\) and \(\beta=1\).  Under the Gibbs law
\(\pi_p(dx)\propto \exp(-F_p(x))dx\), the radial identity gives
\[
  \pi_p(F_p)=\frac{d}{p}.
\]
We therefore report the empirical risk error
\[
  \left|
  \widehat{\mathbb E}F_p-\frac{d}{p}
  \right|.
\]
For stable runs this estimates the stationary risk bias.  When a run leaves
the stable regime, the corresponding entry is interpreted as evidence of the
failure mode of that denominator.

All rows in Table~\ref{tab:fixed-eta-p-scan} use the same stepsize
\(\eta=0.015\), dimension, horizon, burn-in, initialization, seeds, and
threshold rule.  The radial exponent \(p\) is varied to increase tail
stiffness while keeping the numerical design fixed.  The stochastic-gradient
proxy is
\[
  g(x)=\nabla F_p(x)+\sigma\xi,
  \qquad
  \xi\sim N(0,I_d),
  \qquad
  \sigma=3.
\]
The random-denominator row uses \(\|g(x)\|\) in the denominator.  All
deterministic rows use the exact gradient norm in the denominator, while the
update still uses the same stochastic-gradient proxy \(g(x)\).  Thus the
comparison changes the denominator design but not the gradient oracle.

Write
\[
  \bar A(x):=\|\nabla F_p(x)\|.
\]
The global hard and global soft baselines apply their damping everywhere:
\[
  A_{\rm glob}^{\rm hard}(x)=1+\bar A(x),
  \qquad
  A_{\rm glob}^{\rm soft}(x)=1+\bar A(x)^\theta .
\]
The localized rows use the common template
\[
  A_\eta^{c_s,c_h}(x)
  =
  c_s(\bar A(x)-R)_+^\theta
  +
  c_h(\bar A(x)-S)_+,
\]
with
\[
  A_\eta^{\rm soft}=A_\eta^{1,0},
  \qquad
  A_\eta^{\rm hard}=A_\eta^{0,1},
  \qquad
  A_\eta^{\rm hyb}=A_\eta^{1,1}.
\]
The reported runs use
\[
  \alpha=0.45,
  \qquad
  \theta=0.5,
  \qquad
  R=100,
  \qquad
  S=1000.
\]
Thus the localized soft, localized hard, and hybrid rows share the same
activation thresholds; they differ only in whether the soft component, the
hard-tail component, or both components are present.

\begin{table}[H]
\centering
\scriptsize
\caption{Fixed-stepsize radial stress test.}   
\label{tab:fixed-eta-p-scan}
\begin{tabular}{lcccc}
\toprule
Method & \(p=8\) & \(p=10\) & \(p=12\) & \(p=14\) \\
\midrule
Random denominator
& \(8.13\pm0.14\)
& \(21.9\pm0.65\)
& \(85.4\pm4.6\)
& \(374\pm42\) \\
Global hard deterministic
& \(7.55\pm0.13\)
& \(19.2\pm0.76\)
& \(78.0\pm11.0\)
& \(359\pm40\) \\
Global soft deterministic
& \(0.540\pm0.0046\)
& \(0.548\pm0.012\)
& \(0.603\pm0.0098\)
& unstable \\
Localized soft deterministic
& \(0.0928\pm0.0043\)
& \(0.121\pm0.0049\)
& \(0.169\pm0.010\)
& unstable \\
Localized hard deterministic
& \(0.0928\pm0.0062\)
& \((8.77\pm2.70)\times10^5\)
& \((2.82\pm0.18)\times10^9\)
& \((1.17\pm0.05)\times10^{12}\) \\
Hybrid deterministic
& \(0.0909\pm0.0038\)
& \(0.114\pm0.0038\)
& \(0.181\pm0.0070\)
& \(0.322\pm0.021\) \\
\bottomrule
\end{tabular}
\end{table}

Here ``unstable'' means that the repetitions do not produce a reliable averaged risk error because some runs leave the numerical stability range.  

The large and unstable entries show different ways in which a one-component denominator can fail.  The random denominator produces large errors across all exponents because the denominator
is coupled to the stochastic-gradient noise.  The global hard deterministic denominator remains stable, but it damps the drift everywhere and therefore distorts the typical region.  The global soft and localized soft denominators reduce this distortion for moderate exponents, but lose tail control when the
radial growth becomes stiff.  The localized hard denominator has the opposite problem: it provides a strong far-tail safeguard, but it is inactive on the intermediate band \(R<\bar A(x)\le S\).  The hybrid denominator combines the two missing pieces.  Its soft component starts acting before the far tail is reached, and its hard component keeps the far-tail step under control.

To make this mechanism visible, Table~\ref{tab:envelope-range-diagnostic}
reports three scale diagnostics for the hardest case \(p=14\).  For a
deterministic envelope \(A_\eta\), define the effective drift factor
\[
  R_\eta(x)=\frac{1}{1+\eta^\alpha A_\eta(x)}.
\]
The first diagnostic evaluates \(R_\eta(x)\) at the \(99.9\%\) reference
quantile of \(\|\nabla F_p(X)\|\), where \(X\sim\pi_p\).  This measures how
much the denominator compresses the typical region.  The other two diagnostics
evaluate the deterministic drift step
\[
  \eta\|\nabla F_p(x)\|R_\eta(x)
\]
at an intermediate scale \(\|\nabla F_p(x)\|=500\) and at a far-tail scale
\(\|\nabla F_p(x)\|=10^5\).

\begin{table}[H]
\centering
\scriptsize
\caption{Envelope diagnostics for the radial stress test.}
\label{tab:envelope-range-diagnostic}
\begin{tabular}{lccc}
\toprule
Method
& Reference factor
& Intermediate step
& Far-tail step \\
\midrule
Global hard deterministic & \(0.114\) & \(0.0978\) & \(0.0993\) \\
Global soft deterministic & \(0.450\) & \(1.66\) & \(30.7\) \\
Localized soft deterministic & \(1.000\) & \(1.86\) & \(30.8\) \\
Localized hard deterministic & \(1.000\) & \(7.50\) & \(0.100\) \\
Hybrid deterministic & \(1.000\) & \(1.86\) & \(0.100\) \\
\bottomrule
\end{tabular}
\end{table}

The three columns separate the two design quantities used throughout the
experiments.  The reference factor records how much the typical region is
compressed.  The intermediate and far-tail steps show where taming starts to
act and how strong it is after activation.  Global hard taming has strong tail
control, but it also compresses the reference region.  Soft denominators
preserve the reference region and provide intermediate damping, but they leave
a large far-tail step.  Localized hard and hybrid taming have the same
reference factor and the same far-tail step; the difference is the intermediate
band.  The hard-only localized rule is still inactive at
\(\|\nabla F_p\|=500\), whereas the hybrid rule has already activated its
soft component.  Thus the hybrid rule reduces the delayed-activation failure
of the hard-only localized rule without compressing the reference mass.

\subsection{Regression and implementable-envelope examples}\label{subsec:empirical-risk-examples}

\medskip
\noindent\textbf{Experiment 5: Non-radial quartic-regression target.}\par
\smallskip
The radial tests above isolate the denominator and envelope choices in a setting with an analytic reference value.  We next compare ordinary global taming with selective deterministic envelopes in a non-radial empirical-risk problem.  The purpose is to check whether the advantage of the envelope family persists beyond radial examples.

We use the quartic-regression objective in \eqref{eq:quartic-risk}, repeated here for convenience:
\begin{equation}\label{eq:quartic-regression-experiment}
  F_z(w)=\frac1n\sum_{i=1}^n \frac14(a_i^\top w-y_i)^4+\frac\lambda2\|w\|^2 .
\end{equation}
Its gradient depends cubically on the residuals \(a_i^\top w-y_i\), so the drift scale is nonuniform across directions and regions of the state space.  This makes it a natural test case for selective taming: global taming can be stable but conservative, while local or hybrid envelopes can leave moderate-gradient regions less changed and still act when the drift scale is large.

The stochastic-gradient oracle is the mini-batch estimator
\[
  g(w,B)=\frac{1}{|B|}\sum_{i\in B}\nabla_w
  \left[\frac14(a_i^\top w-y_i)^4+\frac\lambda2\|w\|^2\right],
\]
with batches sampled uniformly from the dataset.  In the first run we use \(d=4\), \(n=50\), ridge parameter \(\lambda=0.2\), Gaussian covariates scaled by \(d^{-1/2}\), and mini-batch size \(|B|=6\).

Since no closed-form Gibbs reference is available, we use small-stepsize exact-gradient Langevin runs as numerical references: \(\eta_{\rm ref}=10^{-3}\), 80 parallel chains, 5000 iterations, and 1500 burn-in iterations.  Eight independent reference runs give reference empirical-risk average approximately \(1.32\), with standard error about \(1.4\times10^{-2}\).  The sampled chains use 40 parallel chains, 2500 iterations, and 800 burn-in iterations.  A run is declared unstable if \(\|w_k\|>25\) or if the empirical risk becomes non-finite.

To isolate the denominator-design effect, Table~\ref{tab:quartic-regression-diagnostic} first uses the full-gradient envelope \(\bar A_{\rm diag}(w)=\|\nabla F_z(w)\|\) for deterministic-envelope methods.  Errors of the same order as the combined Monte Carlo uncertainty of the sampled and reference averages should be interpreted as reaching the numerical-reference noise floor rather than as meaningful nonzero stationary bias.

\begin{table}[H]
\centering
\scriptsize
\caption{Non-radial quartic regression: stationary risk averages and errors.}
\label{tab:quartic-regression-diagnostic}
\begin{tabular}{lcccc}
\toprule
Method & \multicolumn{2}{c}{$\eta=0.005$} & \multicolumn{2}{c}{$\eta=0.009$} \\
\cmidrule(lr){2-3}\cmidrule(lr){4-5}
 & Risk mean $\pm$ SE & Error & Risk mean $\pm$ SE & Error \\
\midrule
Random denominator & $2.151\pm0.074$ & $0.829$ & $2.566\pm0.032$ & $1.244$ \\
Global hard deterministic & $1.840\pm0.047$ & $0.518$ & $2.082\pm0.028$ & $0.760$ \\
Global soft deterministic & $1.650\pm0.052$ & $0.329$ & $1.729\pm0.010$ & $0.407$ \\
Local hard deterministic & $1.397\pm0.036$ & $0.075$ & $1.387\pm0.021$ & $0.065$ \\
Hybrid deterministic & $1.358\pm0.028$ & $0.036$ & $1.385\pm0.019$ & $0.063$ \\
\bottomrule
\end{tabular}
\end{table}

\begin{figure}[H]
\centering
\includegraphics[width=0.55\textwidth]{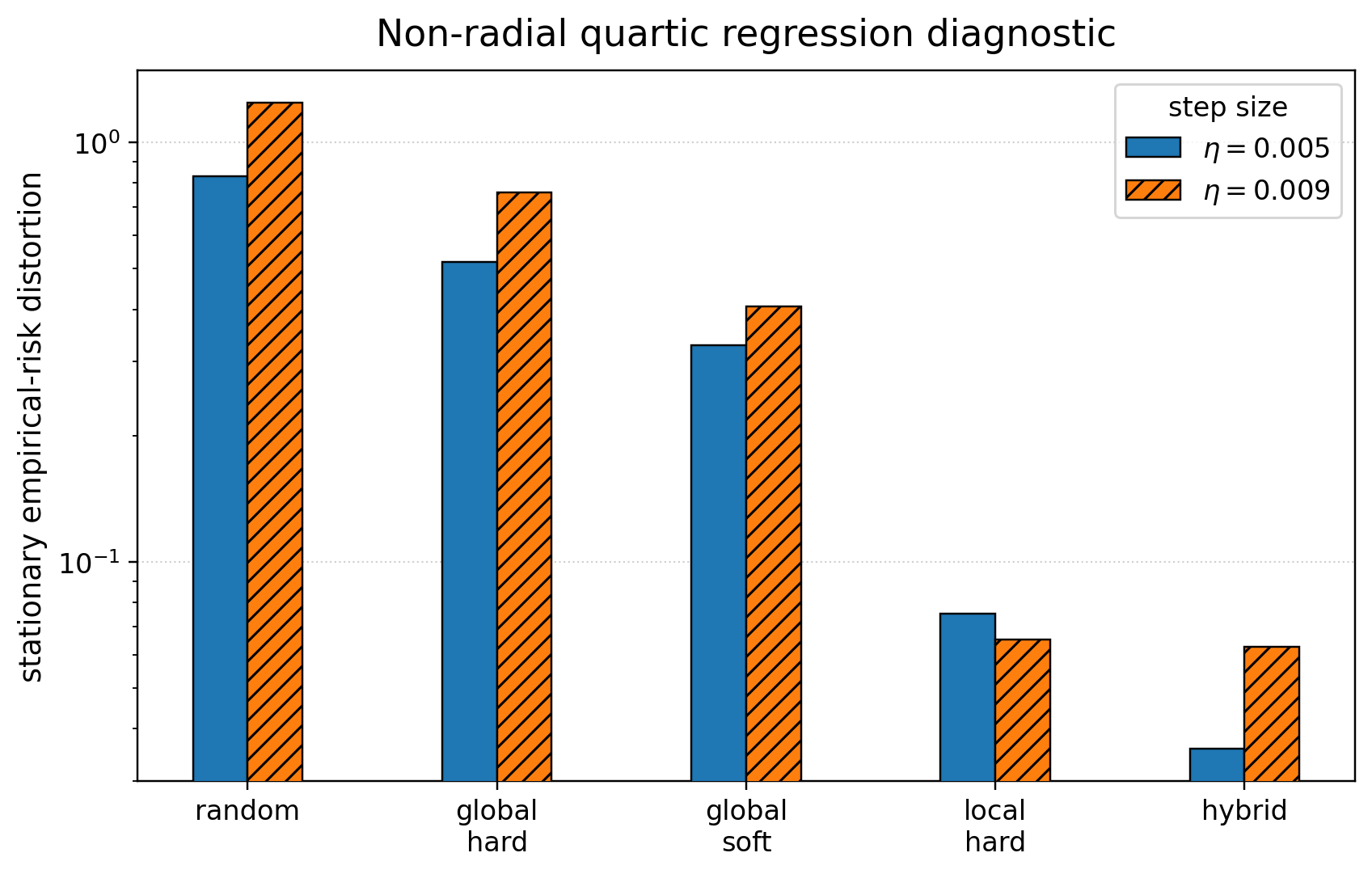}
\caption{Stationary empirical-risk distortion in the non-radial quartic-regression example.}
\label{fig:quartic-regression-diagnostic}
\end{figure}
The reported error is
\(|\widehat{\mathbb E}F_z-F_{\rm ref}|\), where \(F_{\rm ref}\) is this numerical reference average. The reported risk values are stationary averages, not optimized training losses.

The local hard and hybrid rows are close to the numerical-reference noise floor, whereas the random-denominator and global hard rows remain visibly farther from the exact-gradient reference.  This shows that selective deterministic envelopes improve on ordinary global taming in this non-radial example.  In these runs, the local hard and hybrid errors are comparable.  The hard-only rule adds little shrinkage when the sampled paths rarely enter the hard-tail region \(\bar A(x)>S\), while the hybrid rule can also apply a soft correction on the intermediate region \(R<\bar A(x)\le S\).

The same qualitative separation appears in a moderately larger synthetic
instance with \(n=500\), \(d=10\), and mini-batch size \(|B|=20\).  The
reference is again a small-stepsize exact-gradient Langevin numerical average
for the same empirical objective, with estimated reference standard error
about \(9.6\times10^{-2}\).  The local hard and hybrid rows remain closer to
the reference than the random-denominator and global-hard rows.  All listed
runs had zero blow-up.

\begin{table}[H]
\centering
\small
\caption{Quartic-regression check on a moderately larger instance.}
\label{tab:quartic-larger-instance}
\begin{tabular}{lcc}
\toprule
Method & Risk mean $\pm$ SE & Error \\
\midrule
Random denominator & $5.307\pm0.146$ & \(2.10\) \\
Global hard deterministic & $4.814\pm0.060$ & \(1.61\) \\
Global soft deterministic & $4.459\pm0.294$ & \(1.25\) \\
Local hard deterministic & $3.682\pm0.036$ & \(4.77\times10^{-1}\) \\
Hybrid deterministic & $3.861\pm0.049$ & \(6.56\times10^{-1}\) \\
\bottomrule
\end{tabular}
\end{table}

\medskip
\noindent\textbf{Experiment 6: Coarse implementable envelopes and calibration.}\par
\smallskip

The preceding experiments use the full-gradient diagnostic envelope
\(\bar A_{\rm diag}(w)=\|\nabla F_z(w)\|\).  This envelope is useful for
isolating the mechanism, but it is not a low-cost rule for stochastic-gradient
sampling.  As noted in Example~\ref{ex:quartic-risk}, the quartic-regression
drift has cubic growth.  We therefore test the crude deterministic growth
indicator
\[
  \bar A_{\rm poly}(w)=C_z(1+\|w\|^3),
\]
which is independent of the current mini-batch and can be evaluated without
computing the full gradient.  This proxy is not meant to be optimized; the
purpose of the experiment is to see what remains of the selective-envelope
mechanism when the growth scale is only coarsely matched.

Table~\ref{tab:quartic-envelope-implementation} compares the full-gradient
diagnostic envelope with the polynomial proxy.  Since the two envelopes have
different numerical scales, the local and hybrid thresholds are rescaled.  The
same exact-gradient numerical reference as in the preceding \(d=4,n=50\) run is used, with
reference standard error about \(1.4\times10^{-2}\); the reported error is
\(|\widehat{\mathbb E}F_z-F_{\rm ref}|\).

\begin{table}[H]
\centering
\scriptsize
\caption{Diagnostic and polynomial envelopes for quartic regression.}
\label{tab:quartic-envelope-implementation}
\begin{tabular}{llccc}
\toprule
Method & Envelope & Thresholds \((R,S)\) & Risk mean $\pm$ SE & Error \\
\midrule
Global hard & ideal & -- & $1.848\pm0.039$ & $0.526$ \\
Global hard & poly & -- & $43.823\pm3.613$ & $42.501$ \\
Local hard & ideal & \((5,50)\) & $1.388\pm0.008$ & $0.066$ \\
Local hard & poly & \((200,2000)\) & $1.369\pm0.022$ & $0.048$ \\
Hybrid & ideal & \((5,50)\) & $1.372\pm0.024$ & $0.051$ \\
Hybrid & poly & \((200,2000)\) & $1.406\pm0.026$ & $0.084$ \\
\bottomrule
\end{tabular}
\end{table}

The comparison has two messages.  First, using the polynomial proxy as a global
hard denominator creates a large stationary distortion.  Second, the same proxy
works well once its correction is localized.  Thus the polynomial proxy is not
useless as a location signal.  The failure of the global row comes from the
strength of the correction applied on the typical region.

To see this directly, Table~\ref{tab:quartic-poly-shrinkage} reports the
retained-drift factor
\[
  r_\eta(w)=\frac{1}{1+\eta^\alpha A_\eta(w)}
\]
on samples from the reference chain at \(\eta=0.005\).  Larger values of
\(r_\eta\) mean that more of the original drift is retained.  The table
compares the full-gradient scale \(\|\nabla F_z(W)\|\) with the polynomial
scale \(\bar A_{\rm poly}(W)\).  The thresholds for the localized polynomial
rules are \(R=200\) and \(S=2000\).

\begin{table}[H]
\centering
\scriptsize
\caption{Polynomial-envelope scale and drift factors on the reference region.}
\label{tab:quartic-poly-shrinkage}
\begin{tabular}{ccccc}
\toprule
Reference quantile
& \(\|\nabla F_z(W)\|\)
& \(\bar A_{\rm poly}(W)\)
& Global hard poly
& Local/hybrid poly \\
\midrule
50\%   & 1.47 & 22.8  & 0.313 & 1.000 \\
95\%   & 4.01 & 89.1  & 0.107 & 1.000 \\
99\%   & 5.43 & 135.6 & 0.074 & 1.000 \\
99.9\% & 7.41 & 193.6 & 0.053 & 1.000 \\
\bottomrule
\end{tabular}
\end{table}

The diagnostic separates two design quantities.  The first is the activation
location: where the denominator starts to act.  The second is the taming
strength: how much drift is removed after activation.  In this example the
polynomial scale is much larger than the full-gradient scale on the reference
region.  A global polynomial denominator therefore keeps only about \(31\%\)
of the drift at the median and about \(11\%\) at the \(95\%\) quantile.  The
localized polynomial rules avoid this over-shrinkage because the chosen
threshold leaves almost all of the reference region inactive.

A threshold sweep confirms this interpretation.  With \(R=25\), the local hard
polynomial rule is active on about \(46\%\) of the reference region and its
risk rises to about \(37.4\).  With \(R=200\), the active fraction drops below
\(0.1\%\) and the risk returns to the reference scale.  The hybrid rule is
less sensitive to early activation: at \(R=25\) its risk is about \(3.30\), and
by \(R=75\) it is already close to the reference level.  The soft component
therefore moderates the taming strength before the hard-tail correction is
activated.

This experiment marks the practical boundary of the present paper.  The
selective-envelope principle can make even a crude deterministic proxy useful,
but only when activation location and taming strength are calibrated together.
A systematic treatment of proxy construction, scale matching, and threshold
calibration is developed in the companion work~\cite{ZhouChen2026DenominatorDesign}.

\subsection{A high-dimensional stress test}\label{subsec:high-dimensional-pilot}

\medskip
\noindent\textbf{Experiment 7: High-dimensional quartic-regression stress test.}\par
\smallskip

Finally, we test the same quartic-regression structure in a higher-dimensional
setting with \(n=1000\), \(d=50\), mini-batch size \(|B|=50\), and
\(\eta=8\times10^{-4}\).  The numerical reference is an averaged
small-stepsize exact-gradient run for the same empirical objective, computed
from twelve independent reference chains.  The estimated reference empirical
risk is \(8.69\), with standard error about \(8.3\times10^{-2}\).

For the hybrid rule, the thresholds are selected from warm-up quantiles of the
deterministic growth indicator: \(R\approx Q_{0.6}\), \(S\approx Q_{0.99}\),
and \(\theta=0.5\).  The selected thresholds are then held fixed for the
reported production runs.  In Table~\ref{tab:quartic-d50-pilot}, errors are
relative to the averaged exact-gradient numerical reference.  The safe, soft,
and tail columns report the sampled fractions in the three regimes
\[
  \bar A\le R,\qquad R<\bar A\le S,\qquad \bar A>S.
\]

\begin{table}[H]
\centering
\scriptsize
\caption{High-dimensional quartic-regression stress test.}
\label{tab:quartic-d50-pilot}
\begin{tabular}{lccccc}
\toprule
Method & Risk mean $\pm$ SE & Error & Safe & Soft & Tail \\
\midrule
Random denominator & $10.579\pm0.610$ & \(1.89\) & -- & -- & -- \\
Global hard deterministic & $9.662\pm0.382$ & \(9.76\times10^{-1}\) & -- & -- & -- \\
Local hard deterministic & $10.458\pm0.257$ & \(1.77\) & 0.55 & 0.36 & 0.09 \\
Hybrid deterministic & $8.999\pm0.347$ & \(3.13\times10^{-1}\) & 0.63 & 0.36 & 0.01 \\
\bottomrule
\end{tabular}
\end{table}

Table~\ref{tab:quartic-d50-pilot} should be read as a stress test rather than
a high-precision ranking.  The high-dimensional reference has non-negligible
Monte Carlo uncertainty.  The hybrid error, \(3.13\times10^{-1}\), is
comparable to the propagated uncertainty scale, about \(3.57\times10^{-1}\).
Thus, for this quantile-selected configuration, the hybrid row is already at
the numerical-reference resolution of this experiment and gives the smallest
reported error among the tested denominators.  The other reported
denominators remain well outside this scale.

The regime fractions give a useful mechanism check for this run.  The local
hard rule visits the tail regime about \(9\%\) of the time, whereas the hybrid
rule reduces this fraction to about \(1\%\), while keeping a similar
soft-region fraction.  This is consistent with the intended hybrid mechanism:
the soft component acts before frequent hard-tail visits occur.  Thus, in this
particular high-dimensional stress instance, the hybrid choice is the best
among the reported denominator designs, while the broader lesson is that the
useful localized rule can depend on the envelope scale and threshold selection.

\section{Discussion}\label{sec:discussion}

The analysis above used analytic deterministic envelopes.  This choice isolates
the oracle effect of the denominator.  It is not yet a complete
implementation strategy.  In applications, the full drift scale may be costly
or unavailable.  One must then build a cheaper proxy for the region where
taming is needed.  Such a proxy has to match both scale and activation
geometry.  Otherwise, it may preserve the oracle mean but still damp the chain
in the wrong region.  Proxy construction, quantile calibration, and scale
matching are studied in the companion work~\cite{ZhouChen2026DenominatorDesign}.

A second direction is to sharpen the residual theory.  The Poisson argument
bounds stationary bias by an integrated one-step residual.  For deterministic
stochastic-gradient envelopes, this residual has three parts: the Euler
residual, the envelope residual, and the oracle-covariance residual.  This
residual decomposition should be connected with small-stepsize weak
convergence.  Such a theory would make the role of the denominator more
quantitative.  It would show how a scheme allocates weak error between the
typical region and the tail.

A third direction is to apply the same test to existing tamed SGLD schemes.
TULA, TUSLA, clipping, and related transformations all modify the stochastic
oracle seen by the chain.  If the transformation is deterministic conditional
on the state, the main question is the deterministic residual it creates.  If
it depends on the current stochastic-gradient draw, a mean-shift channel can
also appear.  This gives a common way to compare concrete schemes: not only by
stability, but by the residuals they introduce.

\end{document}